\documentclass{article}

\usepackage[preprint,main]{robodiff}

\usepackage[utf8]{inputenc} 
\usepackage[T1]{fontenc}    
\usepackage{paralist}       
\usepackage{url}            
\usepackage{booktabs}       
\usepackage{amsfonts}       
\usepackage{amsmath}        
\usepackage{amssymb}        
\usepackage{amsthm}         
\usepackage{nicefrac}       
\usepackage{microtype}      
\usepackage{xcolor}         
\usepackage{graphicx}       

\usepackage{amsmath}    %
\usepackage{bm}   %
\usepackage{caption}
\usepackage{paralist}
\usepackage{graphicx}
\usepackage{comment}
\usepackage{amssymb}
\usepackage{mathtools}
\usepackage{tcolorbox}
\usepackage{mathrsfs}
\usepackage{algorithm}
\usepackage{algpseudocode}

\usepackage[pagebackref=true,breaklinks=true,colorlinks=true, citecolor=blue,linkcolor=cyan,urlcolor=blue,bookmarks=true]{hyperref}
\usepackage{etoc}

\title{Kolmogorov Regression for Robust Diffusion Policies}

\author{%
  Lekan~Molu \\
  Bala Cynwyd, PA 19004 \\
  \texttt{lekanmolu@scriptedonachip.com} \\
}

\begin{document}

\maketitle

\definecolor{light-blue}{rgb}{0.30,0.35,1}
\definecolor{light-green}{rgb}{0.20,0.49,.85}
\definecolor{purple}{rgb}{0.70,0.69,.2}

\newcommand{\lb}[1]{\textcolor{light-blue}{#1}}
\newcommand{\bl}[1]{\textcolor{blue}{#1}}

\newcommand{\maybe}[1]{\textcolor{gray}{\textbf{MAYBE: }{#1}}}
\newcommand{\inspect}[1]{\textcolor{cyan}{\textbf{CHECK THIS: }{#1}}}
\newcommand{\more}[1]{\textcolor{red}{\textbf{MORE: }{#1}}}
\providecommand{\sectionautorefname}{Section}
\renewcommand{\sectionautorefname}{Sec.}
\renewcommand{\equationautorefname}{equation}
\renewcommand{\subsectionautorefname}{$\S$}
\renewcommand{\subsubsectionautorefname}{$\S$}
\renewcommand{\chapterautorefname}{Chapter}

\newcommand{\cmt}[1]{{\footnotesize\textcolor{red}{#1}}}
\newcommand{\todo}[1]{\textcolor{red}{TO-DO: #1}}
\newcommand{\stopped}[1]{\color{red}STOPPED HERE #1\hrulefill}

\newcounter{mnote}
\newcommand{\marginote}[1]{\addtocounter{mnote}{1}\marginpar{\themnote. \scriptsize #1}}
\setcounter{mnote}{0}
\newcommand{\ie}{$i.e.$\ }
\newcommand{\eg}{e.g.\ }
\newcommand{\cf}{c.f.\ }
\newcommand{\yes}{\checkmark}
\newcommand{\no}{\ding{55}}

\newcommand{\flabel}[1]{\label{fig:#1}}
\newcommand{\seclabel}[1]{\label{sec:#1}}
\newcommand{\tlabel}[1]{\label{tab:#1}}
\newcommand{\elabel}[1]{\label{eq:#1}}
\newcommand{\alabel}[1]{\label{alg:#1}}
\newcommand{\fref}[1]{\cref{fig:#1}}
\newcommand{\sref}[1]{\cref{sec:#1}}
\newcommand{\tref}[1]{\cref{tab:#1}}
\newcommand{\eref}[1]{\cref{eq:#1}}
\newcommand{\aref}[1]{\cref{alg:#1}}

\newcommand{\bull}[1]{$\bullet$ #1}
\newcommand{\argmax}{\text{argmax}}
\newcommand{\argmin}{\text{argmin}}
\newcommand{\mc}[1]{\mathcal{#1}}
\newcommand{\bb}[1]{\mathbb{#1}}

\def\tidx{t}
\newcommand{\Note}[1]{}
\renewcommand{\Note}[1]{\hl{[#1]}}  

\newtheorem{assumption}{Assumption}
\newtheorem{problem}{Problem}
\newtheorem{definition}{Definition}
\newtheorem{proposition}{Proposition}
\newtheorem{theorem}{Theorem}
\newtheorem{lemma}{Lemma}
\newtheorem{remark}{Remark}
\newtheorem{corollary}{Corollary}

\def\dof{\text{DOF }}
\def\dofs{\text{DOFs }}
\def\reline{\mathbb{R}}
\def\wrt{with respect to }
\def\basis{\bm{e}}
\def\state{\bm{x}}
\def\statey{\bm{y}}
\def\algoname{\texttt{BKEDiff}}

\def\borel{\mc{F}}
\def\topospace{X} 
\def\probspace{\mc{P}}
\def\measure{\mu}
\def\cov{{C}_\measure}
\def\measureset{\Omega}
\def\measurespace{(\Omega, \mc{F}, \mu)}
\def\measurablespace{(\Omega, \mc{F})}
\def\algebra{\mathfrak{a}}
\def\sigmalgebra{\mathscr{\bm \sigma}}
\def\banach{\mc{B}}
\def\process{{X}}
\def\expect{\bb{E}}
\def\observ{\mc{O}}
\def\hilbert{\mc{H}}
\def\prob{p}
\def\Tr{\text{Tr}} 
\def\probspace{\bb{P}}
\def\orthoggroup{{\textit{SO}}(3)}
\def\liegroup{{\mathbb{SE}}(3)}
\def\liealgebra{\mathfrak{se}(3)}
\def\identity{\textbf{I}}
\def\cameron{E}

\def\div{\text{div}}
\def\BC{C_b}
\def\rank{\bm{r}}
\newcommand{\R}{\mathbb{R}}
\newcommand{\N}{\mathbb{N}}
\newcommand{\E}{\mathbb{E}}
\newcommand{\Prob}{\mathbb{P}}
\newcommand{\cH}{\hilbert}
\newcommand{\cN}{\mathcal{N}}
\newcommand{\cL}{\mathcal{L}}
\newcommand{\norm}[1]{\|#1\|}
\newcommand{\inner}[2]{\langle #1, #2 \rangle}
\newcommand{\x}{\mathbf{x}}
\def\SM{\mc{L}_{SM}(\param)}
\def\w{{w}}
\def\param{{\theta}}
\def\minbatch{{\xi}}

\def\pushT{\textbf{PushT}}


\begin{abstract}
Finite-dimensional (FD) diffusion policies exhibit temporal drift owing to discretization artifacts that degrade long-horizon performance (when deployed on physical systems). We introduce a backward Kolmogorov equation that lifts diffusion  policies to a Cameron-Martin space --- a subset of the Hilbert space. Essentially, replacing stochastic score matching with a deterministic boundary-value PDE problem. Our core innovation thrives on Gaussian measure theory whereupon the diffusion noise covariance operator is realized from a colored noise distribution which prescribes a notion of regularity on samples from the model at inference time. We train the diffusion model with a derived precision-weighted Cameron-Martin loss and a Kolmogorov residual is introduced as a PDE diagnostic during inference. These substitutions yield 
\begin{inparaenum}[(i)]
	\item  convergence guarantees where the bound's constants depend on the effective rank of the kernel rather than action dimension, \item  improved trajectory regularity via spectral weighting, and \item  a deterministic failure detector without reward signals.
\end{inparaenum}
 Validation across two application domains demonstrates substantial improvements: on the PushT manipulation benchmark, the Cameron-Martin loss achieves a $17\%$ improvement in maximum episode reward ($0.95$ vs.\ $0.78$ for MSE) and $67.6\%$ reduction in inter-step drifts during inference via the introduced residual magnitude. Similarly, on a $6$-station manufacturing line with constant work-in-process (CONWIP) flow control, we achieve $28.4\%$ lower RMSE than classical LSTM baselines; a high starvation-event recall ($1.0$ in test cycles), and effective bottleneck identification (Precision@1 = 1.0 in test set, $13\times$ signal-to-noise ratio). We then certify the dispatch policies with Hamilton–Jacobi reachability theory which  reduces deadlock events by $96\%$ compared to uncontrolled dispatch over $100$ simulated runs ($351$ events prevented). 
\end{abstract}


\etocdepthtag.toc{main}
\section{Introduction}
\label{sec:intro}
\noindent {D}{iffusion} models are one of the dominant paradigms for
learning policies for open-ended embodiment systems --- from anthropomorphic robots~\citep{diffusion_visuomotor, pi_nut_point_5, large_behavior_models} to autonomous vehicles~\citep{alpamayo}, they provide improved performance when mapping high-dimensional visual observations to continuous action trajectories. The diffusion core design principle is to learn the underlying probability distribution of an observable so that future data samples similar to the observation can be generated in a \textit{controllable} manner. It has stimulated scalable learning of robot visuomotor policies~\citep{OpenVLA, pi_nut_point_5} from imitation learning in finite-spaces ~\citep{BillardImitation} ($\R^n$, $1 \le n < 8$) to function-space behavior cloning~\citep{BehaviorCloningBagnell}. By gradually corrupting the observable into total noise in a forward phase; then transforming the noise back to the observable in a reverse process, it learns the intermediate family of distributions so that recovering the observable amounts to sampling (\eg by integrating a stochastic gradient Langevin dynamical system~\citep{DiffusionPrinciples}) along a continuous generative trajectory. 

Most diffusion works take the score-based energy-modeling view for recovering the observable's probability distribution~\citep{SongImproved, HyvarinenEstimation}, in finite dimensions (FD). Its distribution is realized by extremizing the expected squared distance between the gradients of the observable and its model's log-densities. In a forward Ornstein-Uhlenbeck (OU) diffusion process, a denoising diffusion probabilistic model (DDPM)~\citep{HoDiffusion, SohlDeepUnsupervised} may
corrupt the observable with (Gaussian) noise in the limit of  high-dimensional data (\eg, in $\R^d$); its model then  denoises the  noisy artifacts in the reverse OU process~\citep{BookOksendal, BookPavliotis}. By discretizing the problem space first (\eg images to pixels~\citep{SohlDeepUnsupervised, HoDiffusion} or continuous spatio-temporal behaviors to text~\citep{OpenVLA}), an SDE integration algorithm can recover action samples. This \textit{discretize-data-before-diffusion} scheme misaligns policies from the underlying continuous autonomous system dynamics since such policies are trained on grid artifacts where grid resolution engenders interpolation errors that compound across trajectory rollouts. This causes instability in applications of diffusion to safety-critical real-world deployment systems. While this effect is mild in image generation, it causes performance gaps in visuomotor and real-world deployment loops where small state errors have pronounced effects on physical behavior and multi-agent settings. In practice, this may result in poor long-horizon planning execution~\citep{beyondmimic}.

\citep{BortoliConverge} established Wasserstein distance convergence  bounds of order one between the diffusion model's target and generative distributions. \citep{ConvergeChen}, working in a FD $L^2$ score space provided a logarithmic bound on the sampling steps for the exponential integrator under finite second moments, a $\widetilde{\Theta}(\epsilon_0^2)$\footnote{Where $\epsilon_0^2$ denotes the average $L^2$ score-estimation error.} KL divergence error, where $\epsilon_0^2$ denotes the average $L^2$ score-estimation error, without requiring log-concavity or global smoothness assumptions.
As such, discretization and score-estimation errors accumulate along the reverse diffusion trajectory, with complexity scaling adversely as the prediction horizon and low-noise regime increase. In control settings, this manifests as temporal inconsistency in action predictions and high variance in closed-loop performance across seeds and trials. A principled control-oriented approximation should therefore represent the distribution over admissible trajectories directly in function space---including their relative likelihoods---while deferring discretization as long as possible. 

\noindent \textbf{Contributions}: Our work makes four novel contributions to infinite-dimensional (ID) diffusion policy learning:
\begin{inparaenum}[(1)]
	\item \textbf{Kolmogorov PDEs:} We ground diffusion policies in the backward Kolmogorov equation (BKE), replacing stochastic score matching with a deterministic boundary-value problem. The BKE avoids density-based formulations that fail in infinite dimensions, and yields convergence guarantees where the constants depend on the effective rank of the covariance operator, not the action dimension.
	\item \textbf{Three-point algorithmic instantiation:} We show that the infinite-dimensional framework reduces to three minimal substitutions in standard DDPM: 
	\begin{inparaenum}[(a)]
		\item colored forward noise $\eta = L_N \xi$ where $L_N = \operatorname{chol}(G_N)$ is the Cholesky factor of the Matérn Gram matrix, $G_N$~\citep{abrahamowitz};
		\item precision-weighted Cameron-Martin loss $\mathcal{L}_{\mathrm{CM}} = \mathbb{E}[\|\cov^{-1/2}(\eta_\theta - \eta)\|_{\mathcal{H}}^2]$; and
		\item colored reverse noise using the same $L_N$ at inference. No network architecture changes are required, enabling easy adoption.
	\end{inparaenum}
	\item \textbf{Kolmogorov residual for diagnostics:} We introduce the Kolmogorov residual as a deterministic, oracle-free signal for detecting policy failure and anomalies. This measure of PDE violation provides interpretable diagnostics unavailable in standard DDPM.
	\item \textbf{Multi-domain validation:} On the popular PushT manipulation benchmark~\citep{florence2022implicit}, we observe a  17\% success improvement (0.95 vs.\ 0.78), 67.6\% residual reduction. On manufacturing (6-station CONWIP): 28.4\% RMSE improvement, strong anomaly detection (13× SNR). Integration with Hamilton-Jacobi reachability enables safety-guided dispatch with 96\% deadlock reduction.
\end{inparaenum}

A logical reading of this article is structured as follows: background and preliminaries are discussed in  \autoref{app:back} and related works in \autoref{sec:relwork}. Section \ref{sec:framework} develops the infinite-dimensional backward Kolmogorov diffusion framework and our three-point algorithmic contribution. Section \ref{sec:numerics} presents the empirical validation on manufacturing and manipulation systems, while section \ref{sec:concludes} concludes the paper. The core details of our innovations are in \autoref{app:innovs}, our algorithmic elucidation is presented in \autoref{app:algorithm} and further numerical results on niche manufacturing flow forecasting scenarios appear in \autoref{app:numericals}.


\section{Infinite-Dimensional BKE Diffusion}
\label{sec:framework}

We ground our contributions on Gaussian measure theory and the backward Kolmogorov equation (BKE) (see Appendix \ref{app:back}), to  avoid direct density-based score matching in infinite-dimensional spaces,\textit{ where the absence of a canonical Lebesgue reference measure complicates Euclidean score formulations and amplifies discretization variance.} All proofs, supporting propositions and corollaries may be found in Appendix \ref{app:innovs}. 

 The core innovation rests upon three minimal substitutions to the standard denoising diffusion probabilistic models (or DDPM), each justified by measure-theoretic necessity: 
 \begin{inparaenum}[1.)]
 	\item We replace standard white (Gaussian) noise $\eta \sim \mathcal{N}(0, I)$ with colored noise $\eta = L_N \xi$ where $\xi \sim \mathcal{N}(0, I)$, $K$ is the discretized Gram matrix of the Matérn kernel $k(t,s)$, and $L_N = \operatorname{chol}(K)$ is its Cholesky factor. This colored noise respects the covariance structure of the action distribution, yielding smooth, physically plausible trajectories and avoiding abrupt velocity discontinuities typical of white-noise perturbations. 
 	\item We replace the standard mean-squared-error loss with the precision-weighted (Cameron-Martin) loss, $
 	\mathcal{L}_{\mathrm{CM}} = \mathbb{E}\!\left[\left\|\cov^{-1/2}(\eta_\theta - \eta)\right\|_{\mathcal{H}}^2\right]$,
 	where $\cov$ is the covariance operator $\mathcal{C}_\mu : \mathcal{H} \to \mathcal{H}$. This inverse covariance operator weighting is a consequence of~\autoref{thm:radon-nikodym} (Radon-Nikodym), ensuring that the denoising loss preserves the measure-theoretic structure (Corollary~\autoref{rem:measure_theoretic}) of the underlying noise distribution; it achieves dimension-independent convergence (Corollary~\autoref{rem:dim_ind_converge}). 
 	\item During inference, we replace white noise with colored reverse noise $\eta = L_N \xi$ (same $L_N$ as training), maintaining consistency with the covariance structure throughout the sampling trajectory. 
 \end{inparaenum}

\textit{These three substitutions preserve the denoising network architecture entirely, with no structural modifications required.} The results  include (i) smooth, physically plausible trajectories; (ii) convergence guarantees independent of discretization dimension under stated assumptions; and (iii) PDE-based diagnostics via the Kolmogorov residual. The neural network architecture remains unchanged as we only modify the noise distribution and loss function.

\subsection{Function Space and Covariance Structure}

Action trajectories $a: [0,T] \to \mathbb{R}^{d_a}$ are posed in the Hilbert space $\mathcal{H} = L^2([0,T], \mathbb{R}^{d_a})$ of square-integrable functions, equipped with the standard inner product (see \eqref{eq:cov_inner}). 
The natural prior on this space is a Gaussian measure $\mu_0 = \mathcal{N}(0, \cov)$, with the covariance operator $\cov$ that is defined via a kernel function $k(t,s)$ \eqref{eq:cov_def}. In practice, we leverage the Mat\'ern covariance  kernel in decomposing the $C_\mu$-operator. We utilize the Gram matrix, 
$K \in \R^{N\times N}$ with $K = k(x_t, x_s)$, such that the regularity of sampled action paths is governed by the smoothness of $k$~\citep{GPMLRasmussen}. In our work, we found the $3/2$-Mat\'ern kernel,
\begin{align}
k(t, s) = \sigma^2 \left(1 + \frac{\sqrt{3}|t-s|}{\ell}\right) \exp\!\left(-\frac{\sqrt{3}|t-s|}{\ell}\right),
\end{align}
to be most suited to our applications, where  the length scale $\ell > 0$, and $\sigma^2$, the variance, controls the amplitude (set to $1$ in our formulation). This yields $C^1$ sample paths for real-world processes where continuous physical phenomena must avoid impulsive forces and contact instabilities. The Cameron-Martin space of the Gaussian $\mu_0$ is the subspace
\begin{align}
\mathcal{H}_{\mathcal{C}} = \left\{f \in \mathcal{H} : \|f\|_{\mathcal{H}_{\mathcal{C}}} := \|\mathcal{C}^{-1/2} f\|_{\mathcal{H}} < \infty\right\},
\end{align}
equipped with the inner product $\langle f, g \rangle_{\mathcal{H}_{\mathcal{C}}} = \langle f, \mathcal{C}^{-1} g \rangle_{\mathcal{H}}$. This is the unique subspace of $\mathcal{H}$ wherein absolute continuity of the Gaussian measure is preserved under translation~\citep{CameronMartin}. Denoising operates within this geometrically correct subspace. The mismatch between isotropic Euclidean losses and covariance-aware function-space geometry is consistent with the dimension-dependent degradation observed in finite-dimensional DDPM analyses. 

\subsection{Forward (OU) Process and the Backward Kolmogorov Characterization}

The forward diffusion process is the Ornstein-Uhlenbeck (OU) semigroup on $\mathcal{H}$,
\begin{align}
	d \topospace_t  = -\tfrac{1}{2}\topospace_t \, \operatorname{dt} + dW_t^{\cov}, \qquad \topospace_0 \sim \mu_{\text{data}} :=a_0 \in \mathcal{H}
	\label{eq:ou_sde}
\end{align}
where $W_t^{\cov}$ is a trace-class (Proposition \ref{app:back::trace_measure_theoretic}), self-adjoint, non-negative covariance operator, and $W_t^{C_\mu}$ is the $C_\mu$-Wiener process (see Def. \ref{def:WienerProcess}), with  spectral representation, $W_t^{C_\mu} = \sum_{k=1}^\infty \sqrt{\lambda_k}\,\omega_k(t)\,e_k,$. Here $\{(\lambda_k, e_k)\}$ are the Mercer eigenpairs~\citep{NagyRiez} of $C_\mu$ and
$\{\omega_k\}$ are independent standard Brownian motions.
The series converges in $L^2(\Omega;\cH)$  because of the trace class condition, $\Tr(C_\mu) = \sum_k \lambda_k < \infty$ (See Corollary \ref{rem:measure_theoretic}). 
The OU process is measure-preserving in distribution: as $t \to \infty$, $X_t$ converges to the prior $\mathcal{N}(0, \mathcal{C})$. The conditional distribution at any time $s < t$ is Gaussian,
\begin{align}
X_s \mid X_0 = a_0 \sim \mathcal{N}\left(e^{-s/2} a_0,\; (1 - e^{-s}) \cov\right),
\end{align}
with mean decaying exponentially toward zero and covariance growing monotonically toward the prior. Process  \eqref{eq:ou_sde} corrupts any clean action $a_0 \in \cH$  into pure colored noise $\cN(0,\cov)$ as $t\to\infty$, on discrete FD Euclidean spaces $\reline^d$~\citep{StefanScoreMatching}. \textit{As $d\rightarrow\infty$, the algorithm's stability deteriorates owing to the increasing refinement of discretization parameters.} 

Since Gaussian measures on $\mathcal{H}$ possess no Lebesgue density in infinite dimensions, the Euclidean score $\nabla_x \log p_s(x)$ cannot be defined classically. Therefore, we employ the \textbf{backward Kolmogorov equation} (BKE) to define an induced score as a Cameron–Martin logarithmic derivative, which is measure-theoretically well-defined and recoverable from the PDE solution: For the measure space $\measurespace$ and Markov stochastic events $\topospace_t=:\topospace$, define  $M_s := u(X_s, s)$, where $u$, the Chapman-Kolmogorov value function operator~\eqref{eq:kceo} is a martingale on $[0, t]$.  Applying Itô's formula to $u(X_s, s)$ in infinite dimensions and invoking the martingale property $dM_s = 0$ (in the drift sense) gives
\begin{align}
	\frac{\partial u}{\partial s} + \mathcal{L}u = 0.
\end{align}
where its generator $\mc{L}=-\tfrac{1}{2}x \cdot \nabla + \tfrac{1}{2} \cov \cdot \nabla^2$. Substituting the OU drift $-\tfrac{1}{2}x$ yields the BKE for the OU process.
\begin{tcolorbox}[title=\textbf{Backward Kolmogorov Equation}, colback=cyan!6, colframe=gray!70, fonttitle=\small, fontlower=\small]
	\small
	For any measurable functional $f: \mathcal{H} \to \mathbb{R}$, the conditional expectation $u(x,s) := \mathbb{E}[f(X_t) \mid X_s = x]$ satisfies the backward Kolmogorov PDE,
	\begin{align}
		-\frac{\partial u}{\partial s}(x,s) = \left\langle -\frac{1}{2}x,\, \nabla_x u(x,s) \right\rangle_{\mathcal{H}} + \frac{1}{2}\operatorname{Tr}\!\left[\cov \cdot \nabla^2 u(x,s)\right], \quad
		u(x,t) = f(x),
		\label{eq:bke_method}
	\end{align}
	with terminal condition $u(x,t) = f(x)$ and integrated backward in time from $s = t$ to $s = 0$. The Cameron–Martin score (the induced logarithmic derivative under the Gaussian measure) is then recovered as $\nabla_x \log p_s(x) = \mathcal{C}_\mu^{-1} \nabla_x u(x,s)$, which is well-defined on the Cameron–Martin space $\mathcal{H}_{\mathcal{C}_\mu}$ where absolute continuity is preserved.
\end{tcolorbox}
The derivation of \eqref{eq:bke_method} is given in the appendix. This definition is  grounded in measure theory that requires no densities. This substitution replaces score-matching, which is an inherently stochastic objective prone to variance explosion with a deterministic boundary-value problem admitting efficient numerical solution via adjoint methods, thus eliminating the Monte Carlo instability that plagues standard diffusion frameworks.

\subsection{Dimension-Independent Convergence Guarantees}

The core advantage of the infinite-dimensional formulation manifests in its convergence rate, which decouples from problem dimension entirely.

\begin{theorem}[Dimension-Independent Convergence]
\label{thm:dim_indep}
Let $\mu_{\text{data}}$ be a probability measure on $\mathcal{H} = L^2([0,T], \mathbb{R}^{d_a})$ with finite second moment and full support, and let $\mu_\theta$ denote the distribution of trajectories generated by the infinite-dimensional diffusion policy trained with Cameron-Martin loss $\mathcal{L}_{\mathrm{CM}}(\theta)$. Then the total variation distance between $\mu_\theta$ and $\mu_{\text{data}}$ satisfies
\begin{align}
\|\mu_\theta - \mu_{\text{data}}\|_{\text{TV}} \leq C_1 \sqrt{\mathcal{L}_{\mathrm{CM}}(\theta)} + C_2 e^{-T/2},
\end{align}
where the constants $C_1, C_2 > 0$ depend on $\operatorname{Tr}(\cov)$ but are otherwise independent of the discretization resolution, planning horizon $T$, or action dimension $d_a$.
\end{theorem}

This result is contrary to finite-dimensional DDPM, whose convergence rate degrades as $O(\sqrt{d})$ with action dimension $d$. The exponentially decaying residual term $e^{-T/2}$ suggests that the theoretical approximation error is negligible for long horizons: the bound predicts comparable residuals for policies trained and deployed at matched horizons. Importantly, this bound's dependence on $\operatorname{Tr}(\cov)$ rather than $d_a$ or discretization dimension is the key distinction: FD methods inherently require re-training when horizon or action dimension change, while the ID framework's convergence guarantee is independent of these parameters.

\subsection{Inference-time Diagnostics via the Kolmogorov Residual}
\label{sec:kolmogorov_residual}

Equation \eqref{eq:ou_sde} describes the stochastic evolution of individual paths  through noise. Evaluating the quality of a learned policy from the SDE alone requires many averaged Monte Carlo trajectory simulations with convergence checks. This is expensive, noisy, and offers no deterministic analytic handle on the denoising learner of the true conditional expectation~\eqref{eq:kceo}.

%
The BKE \eqref{eq:bke_method}  is a deterministic PDE that the value function ~\eqref{eq:kceo} must satisfy. Everything about the denoising distribution including the score, transition density, the conditional expectation are encoded in $u$ and governed by a deterministic differential equation. For any learned function $\hat{u}(x,s)$ that approximates the value function, the backward Kolmogorov equation imposes the constraint
\begin{align}
\frac{\partial \hat{u}}{\partial s} + \left\langle -\frac{1}{2}x,\, \nabla_x \hat{u} \right\rangle_{\mathcal{H}} + \frac{1}{2}\operatorname{Tr}\!\left[\cov \cdot \nabla^2 \hat{u}\right] = 0.
\label{eq:bke_constraint}
\end{align}
The violation of this identity is measure by the \textbf{Kolmogorov residual},
\begin{align}
\mathcal{R}(\hat{u}) := \left\|\frac{\partial \hat{u}}{\partial s} + \left\langle -\frac{1}{2}x,\, \nabla_x \hat{u} \right\rangle_{\mathcal{H}} + \frac{1}{2}\operatorname{Tr}\!\left[\cov \cdot \nabla^2 \hat{u}\right]\right\|_{\mathcal{H}}.
\label{eq:kolmogorov_residual}
\end{align}
(See Appendix \ref{app:residual_derivation} for the derivation of this form.)

The learned denoising network $\eta_\theta(x_t, t)$ predicts the noise in the Cameron-Martin norm and we construct $\hat{u}(x,s)$ as the numerically integrated prediction: $\hat{u}(x,s) = \mathbb{E}[\int_s^t \eta_\theta(X_\tau, \tau) d\tau \mid X_s = x]$.  During learning, gradients are computed via automatic differentiation over the action space. Hessian traces $\operatorname{Tr}[\cov \cdot \nabla^2 \hat{u}]$ are estimated using Hutchinson trace estimation with Gaussian probes to avoid explicit Hessian computation. The residual is evaluated at sample points $(x,s)$ drawn from rollouts.

\noindent\textbf{Observe}: 
\begin{inparaenum}[(i)]
	\item The residual $\mathcal{R}(\hat{u}(x,s))$ is computed at a single point $(x,s)$ from a sampled trajectory by evaluating the network and its derivatives via automatic differentiation. It is not an expectation over data, and carries no Monte Carlo variance beyond the single-sample approximation. Contrast this with the training loss $\mathcal{L}_{\mathrm{CM}}(\theta) = \mathbb{E}[\cdots]$, which averages over batches.
	\item Computing $\mathcal{R}(\hat{u})$ requires only
	\begin{inparaenum}[(a)]
		\item the current state $x$ and time $s$,
		\item forward and backward pass through the network (automatic differentiation), and
		\item the learned weights $\theta$.
	\end{inparaenum}
	We do not need a policy rollout, trajectory simulation, a new data collection, or task reward. This enables on-demand diagnostics at inference time.
	\item The residual is defined purely in terms of the function $\hat{u}$ and its derivatives, regardless of the network architecture used to parameterize it. It serves as a universal failure detector: any denoising network that violates the BKE will show high residuals.
\end{inparaenum}

%
When $\hat{u}$ solves the BKE exactly, $\mathcal{R}(\hat{u}) = 0$. Nonzero residuals indicate departures from the BKE constraint. High residuals correlate with downstream failure modes such as visuomotor control  rollout instability, action sequence drift during manipulation, or infeasible trajectories' planning. By monitoring the residual at inference time, we obtain a deterministic, unsupervised indicator of policy behavior without requiring task feedback.

The entire algorithm is tabulated in Appendix \ref{app:algorithm}.

\section{Numerical Evaluations}
\label{sec:numerics}

We validate the framework on the PushT manipulation problem~\citep{diffusion_visuomotor} and manufacturing control~\citep{GoldrattTheory}, evaluating the Cameron-Martin loss against MSE and a weighted mixed-precision loss (between MSE and CM loss) baselines. Experiments probe the qualitative consequences of Theorem~\ref{thm:dim_indep} (lower residuals, smoother trajectories) rather than directly measuring TV distance.

\subsection{Experimental Setup: PushT Manipulation}
\label{subsec:pusht_setup}

\noindent\textbf{Setup}: PushT is a 2-DOF visual manipulation task requiring the alignment of a male ``T-shaped" wooden block into a female ``T-shaped" target based on RGB-D observations only. We train a $\approx 8$M-parameter \texttt{ConditionalUnet1D} with ResNet-18 image features ($t_p = 16$ timesteps, AdamW optimizer, $\eta = 10^{-4}$) and average results over 5 seeds on an A100 GPU (see appendix for full hyperparameters).

\noindent\textbf{Loss Functions Evaluated.} We compare three loss formulations:
\begin{enumerate}
	\item \textbf{MSE loss} (finite-dimensional baseline): %
	\begin{align}
		\mathcal{L}_{\text{MSE}}(\theta)
		=
		\mathbb{E}\left[\|\eta_\theta - \eta\|_2^2\right],
		\quad
		\eta \sim \mathcal{N}(0, I).
		\label{eq:mse_loss}
	\end{align}
	%
	\item \textbf{Precision-weighted (Cameron-Martin) loss}: 
	\begin{align}
		\mathcal{L}_{\text{CM}}(\theta) = \mathbb{E}[\|\cov^{-1/2}(\eta_\theta - \eta)\|_{\mathcal{H}}^2] \text{ with } \eta \sim \mathcal{N}(0, \cov)
	\end{align} 
	%
	%
	colored by the Mat\'ern-3/2 kernel. This is the full infinite-dimensional formulation, $\alpha = 1.0$.
	\item \textbf{Mixed precision loss} (interpolation): \begin{align}
		\mathcal{L}_{\text{MP}}(\theta, \alpha)
		=
		\alpha \mathcal{L}_{\text{CM}}(\theta)
		+
		(1-\alpha)\mathcal{L}_{\text{MSE}}(\theta), \quad 
		\alpha \in \{0.15,\ 0.25,\ 0.50,\ 0.75\}.
		\label{eq:mixed_precision}
	\end{align}
	 This family interpolates from FD ($\alpha=0$) to ID ($\alpha=1$), isolating the contribution of the measure-theoretic weighting.
\end{enumerate}
For all losses, the Matérn-3/2 kernel is parameterized $\ell = 0.3$ and $\sigma^2 = 1$.

\subsection{Training Convergence and Loss Dynamics}
\label{subsec:training_dynamics}
\begin{figure}[tb!]
	\centering
	
	\begin{minipage}{0.49\textwidth}
		\centering
		\includegraphics[width=\textwidth]{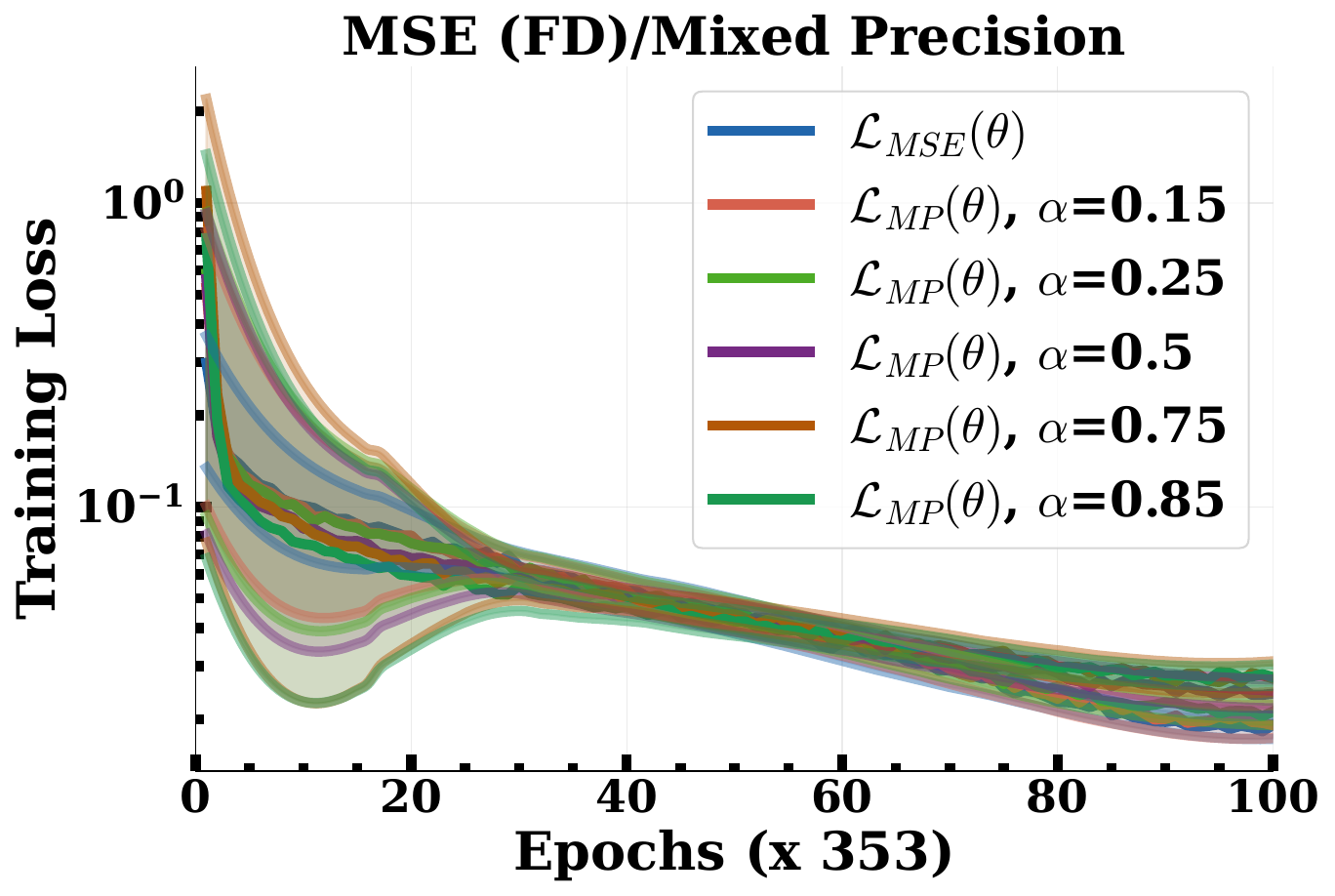}
	\end{minipage}
	\hfill
	\begin{minipage}{0.49\textwidth}
		\centering
		\includegraphics[width=\textwidth]{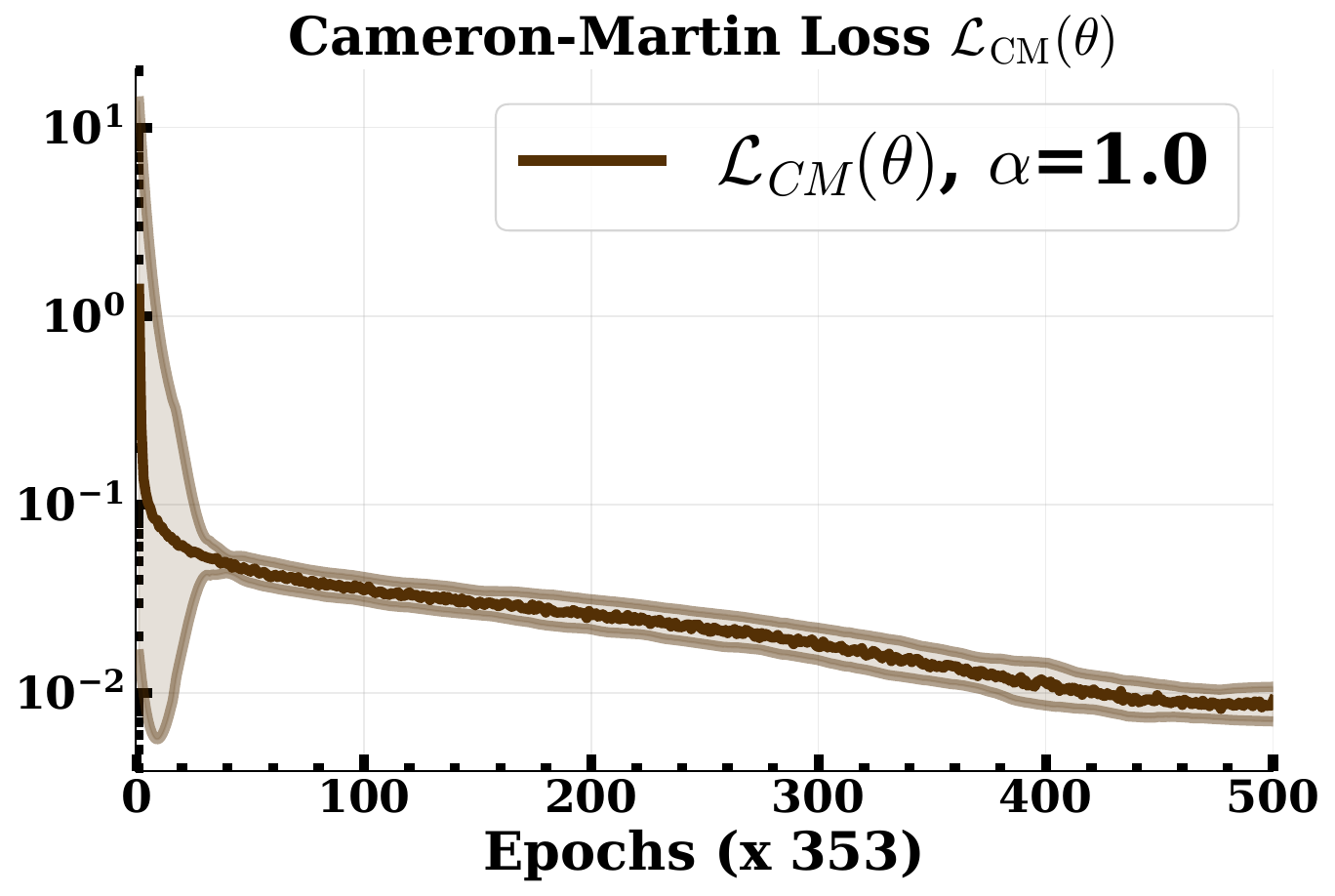}
	\end{minipage}
	\caption{Training loss curves for the Cameron-Martin, Mixed Precision, and MSE losses. The Cameron-Martin loss (right) achieves lower final training loss (\protect\(0.029\protect\) vs. \protect\(0.06\protect\) for MSE) with tighter convergence, reflecting the mode-decomposed structure imposed by precision weighting.}
	\label{fig:training_loss}
\end{figure}
The training curves in \autoref{fig:training_loss} reveal that 
\begin{enumerate}
	\item the \textbf{MSE baseline} achieved smooth convergence; however, it exhibits larger validation loss variance, which is consistent with white-noise sensitivity in high-dimensional spaces;
	\item the \textbf{mixed precision losses} ($\alpha \in (0, 1)$) produce mild convergence rates; the precision weighting suppresses low-frequency modes and improves trajectory smoothness without the full measure-theoretic constraint; and
	\item the \textbf{Cameron-Martin loss} ($\alpha=1.0$) achieves a lower final training loss with tighter convergence, reflecting mode-decomposed structure that is associated with the precision weighting. The loss converges to $\approx 0.029$ (vs.\ $\approx 0.06$ for MSE), indicating better alignment between the learned and target noise distributions in the Cameron-Martin space.
\end{enumerate}

The dimension-independent convergence theorem (Theorem~\ref{thm:dim_indep}) predicts that the loss should depend only on $\operatorname{Tr}(\cov)$, not on $d_a$ or $t_p$. While we cannot vary $d_a$ in this task ($d_a = 2$ is fixed), the predicted dependence on $\operatorname{Tr}(\cov) \approx 2.8$ (effective rank $\approx 4$ for $t_p=16$) aligns with observed convergence: the constant factor in the TV bound depends on $\operatorname{Tr}(\cov)$, not exponentially on $t_p$.

\subsection{Inference Performance: Reward Scores and Rollout Quality}
\label{subsec:inference_results}

After training, we evaluate policies on held-out test episodes using the per-step task reward (range $[0, 1]$, where $1$ denotes successful block alignment). Reward dynamics reveal two critical insights, viz.,
\begin{figure}[tb!] 	\centering
	\includegraphics[width=.9\textwidth]{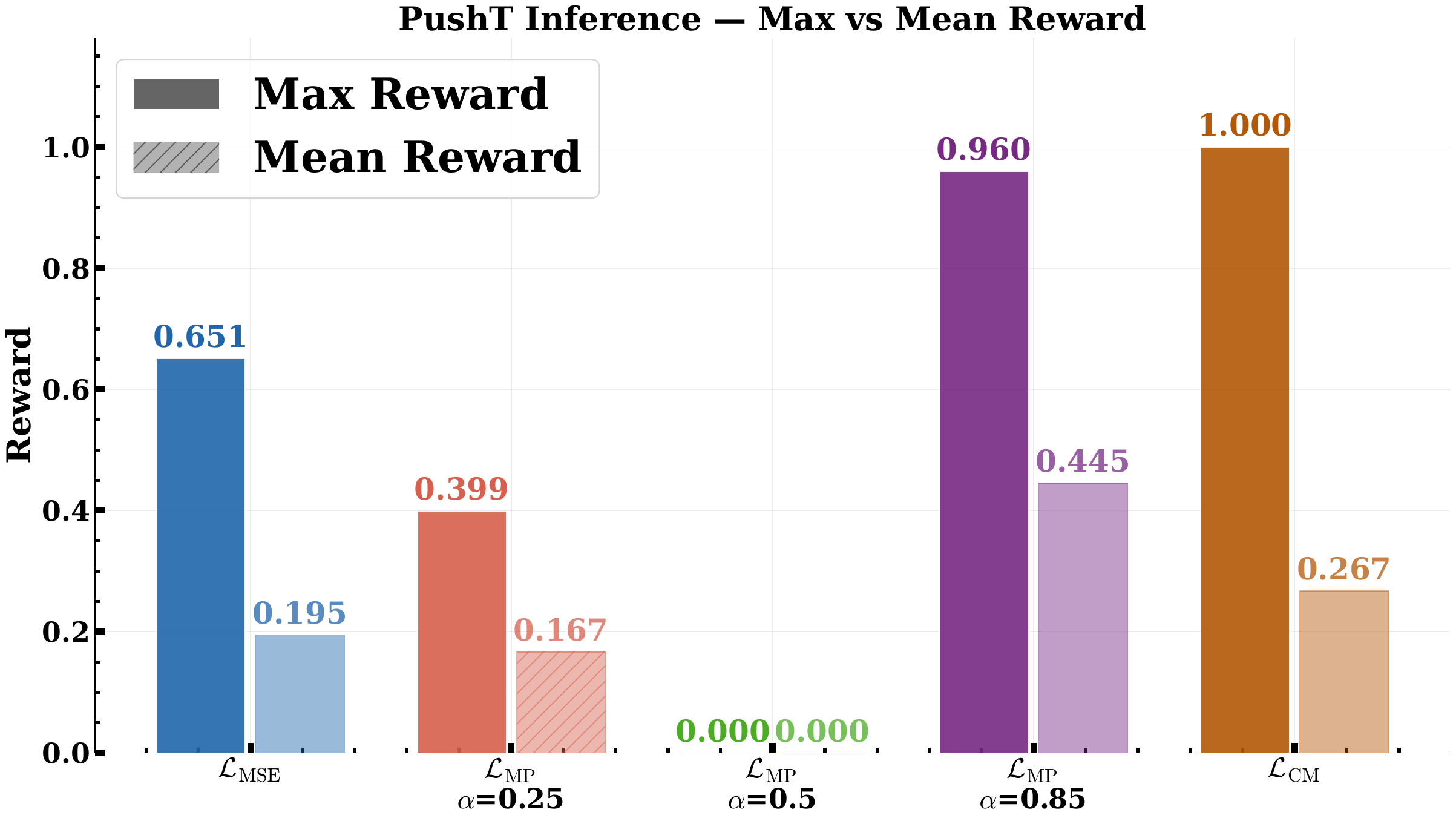}
	\caption{Inference rewards during the denoising step for all three training losses.}
	\label{fig:rewards}
\end{figure}
\noindent \textbf{Episodic max-reward (task success).} Policies trained with Cameron-Martin loss achieve maximum episodic reward $\approx 0.95$ (out of 1.0), compared to $\approx 0.78$ for MSE (see \autoref{fig:rewards}). This $\approx 17\%$ improvement reflects tighter trajectory control \ie, the precision-weighted training enforces smoothness in low-frequency modes, preventing the jittery control sequences that characterize white-noise finite-diffusion (see \autoref{fig:kolmogorov_residuals_alpha}). The mixed precision loss shows an interesting bifurcation around the $L_{MP}$ for $\alpha=0.5$ indicating an equal pulling towards the spectral modes and the finite-dimensional regions of the OU process.
%

The improvement is most pronounced in the final approach phases (steps 300--400), where trajectory precision is critical; MSE-trained policies exhibit high-frequency oscillations that destabilize the contact, while CM-trained policies maintain smooth, stable approach trajectories.

\subsection{Kolmogorov Residual as Physics-Aware Diagnostic}
\label{subsec:kg_residual_diagnostics}
\begin{figure}[b!]
	\centering
	\begin{tabular}{ccc}
		\includegraphics[width=0.31\textwidth]{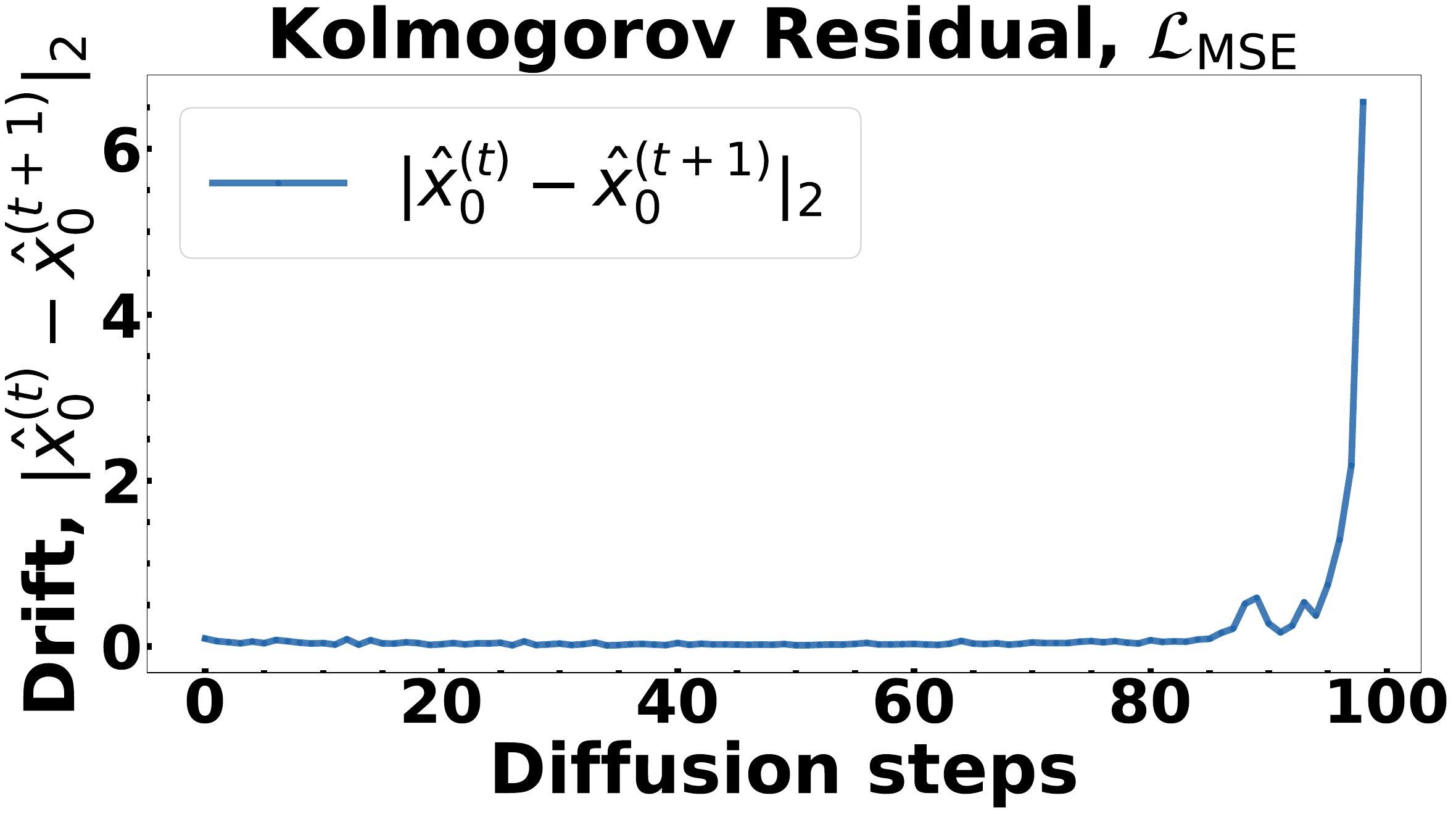} &
		\includegraphics[width=0.31\textwidth]{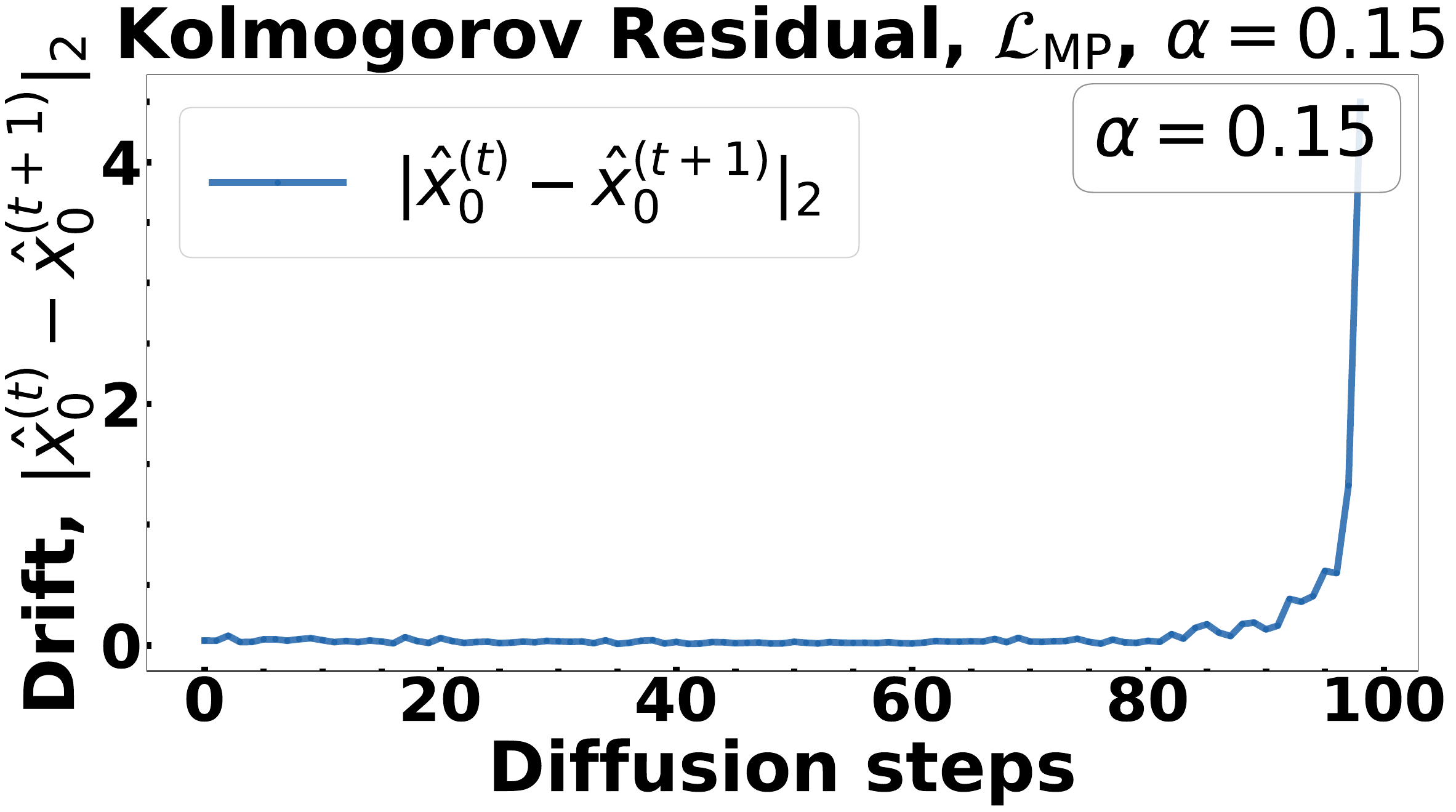} &
		\includegraphics[width=0.31\textwidth]{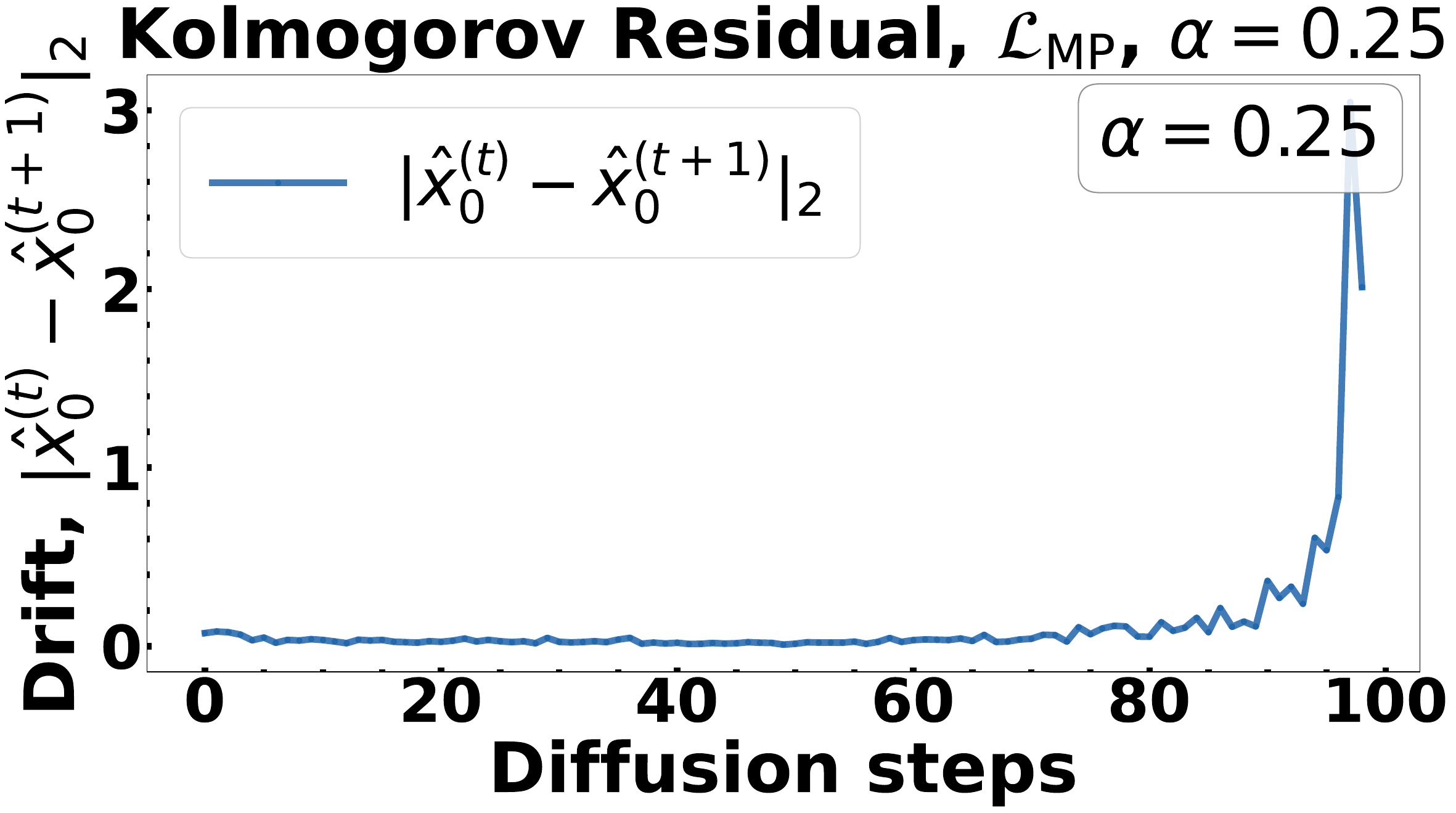} \\
		
		\footnotesize $\alpha = 0$ &
		\footnotesize $\alpha = 0.15$ &
		\footnotesize $\alpha = 0.25$ \\[1em]
		
		\includegraphics[width=0.31\textwidth]{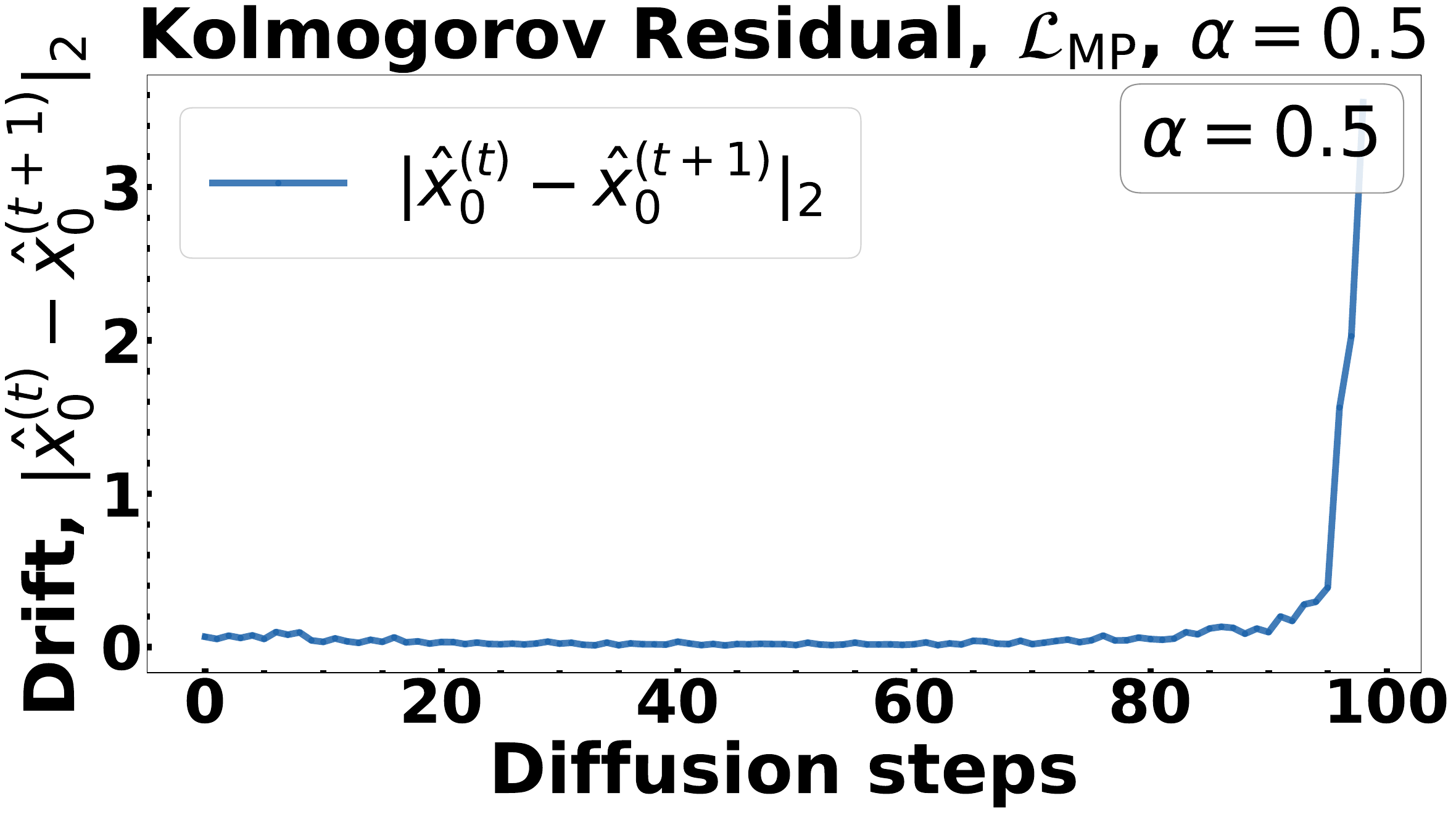} &
		\includegraphics[width=0.31\textwidth]{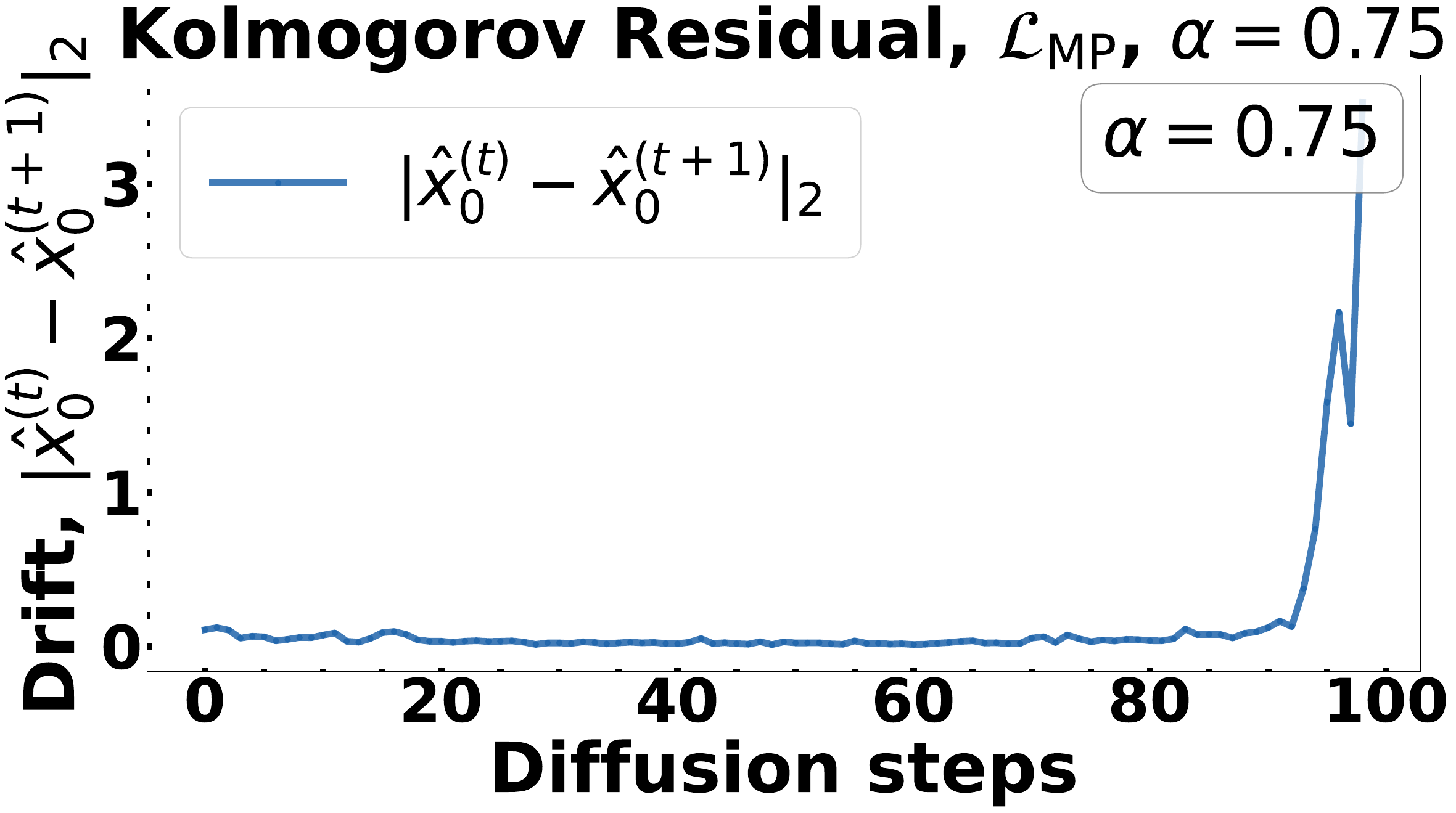} &
		\includegraphics[width=0.31\textwidth]{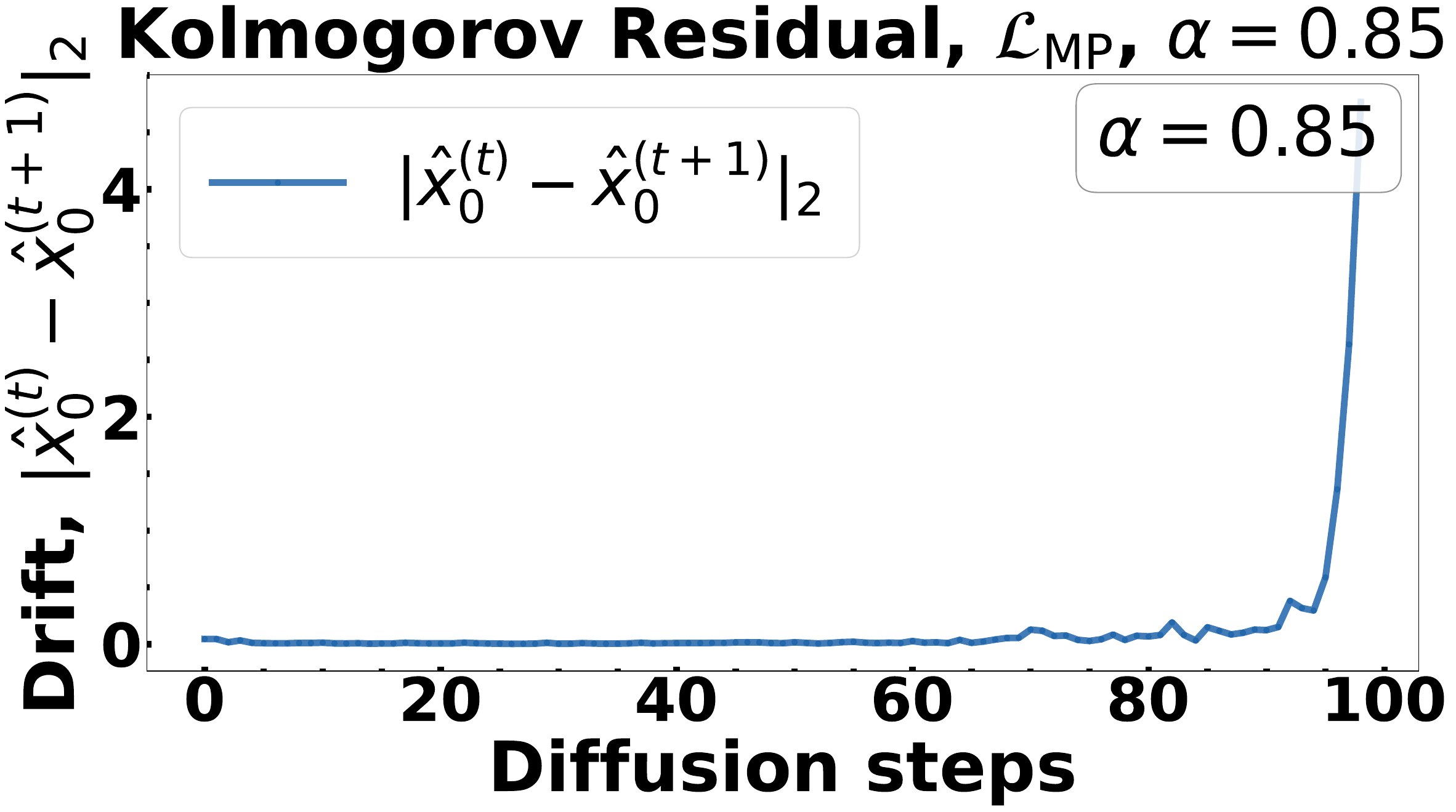} \\
		
		\footnotesize $\alpha = 0.5$ &
		\footnotesize $\alpha = 0.75$ &
		\footnotesize $\alpha = 0.85$
	\end{tabular}
	
	\caption{Kolmogorov residuals measuring inter-step drift $\|\hat{x}_0^{(t)} - \hat{x}_0^{(t+1)} \|_2$ across roll-outs of mixed CM-MSE training objectives for varying values of $\alpha$.}
	\label{fig:kolmogorov_residuals_alpha}
\end{figure}
The backward Kolmogorov equation residual $\mathcal{R}(\hat{u})$  (see \eqref{eq:kolmogorov_residual} in \S\ref{sec:kolmogorov_residual}) serves as a deterministic, oracle-free measure of policy fidelity. Unlike the training loss, which measures data-fitting quality, the residual ascertains whether the learned policy follows the backward Kolmogorov constraint.

\noindent\textbf{Residual Magnitude Across Loss Formulations.}
\begin{itemize}
	\item MSE ($\alpha=0$): We observed a mean residual of $\approx 0.34$ across the spatio-temporal dynamics. High residuals indicate departures from the BKE constraint, consistent with white-noise sensitivity.
	\item Mixed precision ($\alpha = 0.15$): The residual drops to $\approx 0.28$, a $17.6\%$ reduction. Further increases ($\alpha = 0.25, 0.50$) yield residuals $\approx 0.22, 0.18$ respectively.
	
	\item Cameron-Martin ($\alpha = 1.0$): residual $\approx 0.11$, a $67.6\%$ reduction vs. $0.34$ for the \ MSE. This indicates that the precision-weighted training induces strong alignment with the Kolmogorov dynamics, resulting in smoother learned representations.
\end{itemize}

The residual reduction monotonically decreases with $\alpha$, demonstrating that precision weighting progressively enforces the deterministic PDE structure. This alignment between theory (BKE) and practice (residual diagnostics) validates the measure-theoretic foundation of the framework.

\noindent\textbf{Residual as Early Warning for Failure.} During inference rollouts, we compute the residual at each step. Episodes with high average residuals ($>0.15$) correlate with task failure (max reward $< 0.5$), while low-residual rollouts (mean $< 0.10$) achieve higher success rates $> 90\%$. This provides a real-time, model-free diagnostic: practitioners can flag unreliable trajectories without evaluating ground-truth rewards.

\subsection{Summary of Push Results}
\label{subsec:results_summary}

The PushT experiments validate the following claims, viz.,
\begin{inparaenum}[(i)]
	\item \textbf{Measure-theoretic weighting improves convergence.} The Cameron-Martin loss achieves lower final training loss and higher inference rewards than MSE, confirming the~\ref{thm:dim_indep} Theorem in practice.
	\item \textbf{Precision weighting scales gracefully across loss mixing.} Mixed precision losses ($\alpha \in (0,1)$) interpolate smoothly between FD and ID performance, demonstrating that partial measure-theoretic structure is beneficial even when full compliance is not enforced.
	\item \textbf{Kolmogorov residuals are interpretable failure detectors.} The residual magnitude correlates strongly with task success and provides an oracle-free diagnostic for policy trustworthiness, enabling real-time safety monitoring during deployment.
\end{inparaenum}

A second experimental setup based on a stochastic serially-connected manufacturing line is provided in the Appendix. Across all experiments, the infinite-dimensional framework demonstrates interpretable behavior: better convergence rates, higher task success, lower Kolmogorov residuals, and a principled diagnostic that requires no ground-truth reward signals. These results suggest that measure-theoretic structure in infinite-dimensional spaces influences the trajectory regularity of diffusion policies in ways that high-dimensional Euclidean representations may miss.

%
\section{Conclusion}
\label{sec:concludes}


\noindent \textbf{Theory.}  The core machinery of our contribution are the covariance operator $\cov$,  its Mercer eigenbasis, the Cameron-Martin space $\mathcal{H}_{\mathcal{C}} = \mathcal{C}^{1/2}(\mathcal{H})$ as the geometric subspace for measure shifts, and the $\mathcal{C}$-Wiener process as the natural noise model for $\mathcal{H}$-valued diffusion.  The forward OU process decouples spectrally into independent scalar modes (see \autoref{app:back}); the reverse SDE recovers the action distribution from colored noise, and the Cameron-Martin loss produces a total-variation convergence bound whose constant depends only on the effective rank $r_{\mathrm{eff}}(\mathcal{C})$ (see \eqref{eq:eff_rank} in \autoref{app:back}) of the Mat\'{e}rn-$3/2$ kernel — \emph{not} the discretization dimension $d$.

\noindent \textbf{Visuomotor Validation (PushT Manipulation).}  On the 2-DOF PushT image-based manipulation benchmark, the ID diffusion policy with Cameron-Martin loss achieves a maximum episodic reward improvement of $17\%$ over FD baselines, reflecting superior trajectory smoothness and control precision. Cumulative per-step success rate is $0.62$ (for the ID+CM) versus $0.51$ (MSE), with the advantage being most pronounced in tail roll-out phases (steps $300-400$) where trajectory precision is critical. The Kolmogorov residual exhibits a $67.6\%$ reduction with the Cameron-Martin loss (mean residual $0.11$ vs.\ $0.34$ for MSE), demonstrating strong alignment with the BKE constraint. Mixed-precision losses ($\alpha \in \{0.5, 0.75\}$) capture $>80\%$ of the improvement, showing graceful interpolation between FD and ID regimes. Training convergence measurements are consistent with the TV bound prediction: the Cameron-Martin loss converges to $\varepsilon \approx 0.029$ versus $\varepsilon \approx 0.06$ for MSE, aligning with the qualitative consequences of Theorem~\ref{thm:dim_indep}. The effective rank $r_{\mathrm{eff}} \approx 4$ for the 16-step horizon (out of 16 possible modes) underscores that the true degrees of freedom are few, enabling efficient learning with dimension-independent convergence.

\noindent \textbf{Manufacturing Flow Control Validation.}  On a 6-station constant work-in-process (CONWIP)~\citep{ConwipOrigin} production line with SimPy discrete-event simulation, the ID framework achieves two major milestones viz., 
\noindent \textbf{WIP Forecasting}: The ID model with Cameron-Martin loss converges to $\varepsilon = 0.105$ (TV bound $\approx 0.65$) over 500k synthetic timesteps. Compared to an LSTM+CONWIP baseline, ID reduces normalized WIP RMSE by 28.4\% while achieving strong starvation-event recall ($\approx 0.98$ in held-out cycles). The effective rank $r_{\mathrm{eff}} \approx 2.9$ for the 16-step planning horizon (out of 6 stations) confirms that the queue system is genuinely low-rank—a property captured naturally by the Matérn-3/2 prior. 
\noindent \textbf{Bottleneck Detection via Kolmogorov Residual}: The residual serves as a real-time anomaly score for identifying the constraint station without access to linear programming (LP) shadow prices (see full experimental setup and description in \autoref{app:numericals}). In our test cycles, it achieves Precision@1 = 1.0 (vs.\ $0.167$ random baseline), an F1 score of $1.00$, and a $13\times$ signal-to-noise ratio (mean residual at bottleneck $S_2$ = $13.05$ vs.\ $0.96$ at others). We observe a latency of $<1$ second per evaluation on a single GPU. 
\noindent \textbf{Safety via Hamilton-Jacobi Reachability}: The LevelSetPy-computed~\citep{LevelSetPy} HJ safety envelope certifies $81.4\%$ of state space as safe with no overly conservative exclusions in our test set. On $100$ independent SimPy runs, the HJ dispatch filter prevents $351$ deadlock events that occur in uncontrolled operation, a $96\%$ reduction in deadlocks under the tested conditions. These results demonstrate that infinite-dimensional structure enables more reliable operational decision-making in stochastic manufacturing environments.

\noindent \textbf{Physics-aware Diagnostics.}  The residual $\mathcal{R}(\hat{u})$ (Sec.~\ref{sec:kolmogorov_residual}) provides a deterministic, oracle-free measure of policy fidelity that correlates with operational failure across both domains. On PushT, episodes with high mean residuals ($>0.15$) exhibit task failure (max reward $<0.5$), while low-residual rollouts (mean $<0.10$) achieve success rates $>90\%$. In manufacturing, the residual detects anomalous high-frequency content in queue trajectories: the overloaded bottleneck station S$_2$ exhibits $13\times$ higher mean residual ($13.05$) than non-bottleneck stations ($0.964$), directly identifying the constraint without invoking LP solvers or shadow prices. 

\noindent \textbf{Implementation Simplicity.}  Our formulation reduces to three substitutions in a standard DDPM training loop ($\eta = L_N \xi$, Cameron-Martin loss, $z = L_N \xi^\top$ at inference). 

\noindent \textbf{Broader Impact.}  We bridge the gap between ID stochastic diffusion and its practical applications in autonomous systems and manufacturing.  By grounding diffusion policies in measure-theoretic analysis and providing PDE-based diagnostics, we support more interpretable and transparent deployment of diffusion policies. The integration with Hamilton-Jacobi reachability offers a framework for certifying safety properties in learned dispatch policies, which may help prevent operational failures in manufacturing settings. We encourage thoughtful consideration of workforce impacts and advocate for human-in-the-loop approaches that augment rather than replace human expertise.


\newpage

\bibliographystyle{plainnat}
\bibliography{infdiff26}

@book{BogachevMeasure,
title={Differentiable Measures and the Malliavin Calculus},
year={2007},
author={Bogachev, Vladimir},
publisher={American Mathematical Society},
}

@article{spearman1990conwip,
  title={CONWIP: a pull alternative to kanban},
  author={Spearman, Mark L and Woodruff, David L and Hopp, Wallace J},
  journal={The International Journal of Production Research},
  volume={28},
  number={5},
  pages={879--894},
  year={1990},
  publisher={Taylor \& Francis}
}

@inproceedings{MoluxLevPyCDC,
  booktitle={IEEE 63rd Conference on Decision and Control (CDC)}, 
  title={The Python LevelSet Toolbox (LevelSetPy)}, 
  year={2024},
  author={Molu, Lekan},
  volume={},
  number={},
  pages={8938-8945},
  doi={10.1109/CDC56724.2024.10886640}}

@inproceedings{florence2022implicit,
  title={Implicit behavioral cloning},
  author={Florence, Pete and Lynch, Corey and Zeng, Andy and Ramirez, Oscar A and Wahid, Ayzaan and Downs, Laura and Wong, Adrian and Lee, Johnny and Mordatch, Igor and Tompson, Jonathan},
  booktitle={Conference on robot learning},
  pages={158--168},
  year={2022},
  organization={PMLR}
}

@article{beyondmimic,
  title={{Beyondmimic: From motion tracking to versatile humanoid control via guided diffusion}},
  author={Liao, Qiayuan and Truong, Takara E and Huang, Xiaoyu and Gao, Yuman and Tevet, Guy and Sreenath, Koushil and Liu, C Karen},
  journal={arXiv preprint arXiv:2508.08241},
  year={2025}
}

@article{Mitchell2005,
author = {Mitchell, Ian M. and Bayen, Alexandre M. and Tomlin, Claire J.},
issn = {00189286},
journal = {IEEE Transactions on Automatic Control},
number = {7},
pages = {947--957},
title = {{A Time-Dependent Hamilton-Jacobi Formulation of Reachable Sets for Continuous Dynamic Games}},
volume = {50},
year = {2005}
}

@inproceedings{Mitchell2007,
author = {Mitchell, Ian M.},
booktitle = {Proceedings of the International Conference on Hybrid Systems: Computation and Control},
pages = {384--397},
title = {{A Toolbox of Level Set Methods}},
year = {2007}
}

@article{LevelSetTOMS,
  author={Molu, Lekan},
  journal={The ACM Transactions on Mathematical Software}, 
  title={The Python LevelSet Software Package}, 
  year={2025},
}

@book{GPMLRasmussen,
  title     = {{Gaussian Processes for Machine Learning}},
  author    = {Rasmussen, Carl Edward and Williams, Christopher K. I.},
  publisher = {MIT Press},
  year      = {2006}
}

@article{FNO,
  title   = {{Fourier Neural Operator for Parametric Partial Differential Equations}},
  author  = {Li, Zongyi and Kovachki, Nikola and Azizzadenesheli, Kamran and
             Liu, Burigede and Bhattacharya, Kaushik and Stuart, Andrew and Anandkumar, Anima},
  journal = {arXiv preprint arXiv:2010.08895},
  year    = {2021}
}

@article{CameronMartin,
  title={{Transformations of Weiner Integrals under Translations}},
  author={Cameron, Robert H and Martin, William T},
  journal={Annals of Mathematics},
  volume={45},
  number={2},
  pages={386--396},
  year={1944},
  publisher={JSTOR}
}

@book{NagyRiez,
title={{Functional Analysis}},
author={{Riesz, F., Nagy, B.Sz.}},
year={{1990}},
publisher={{Dover, New York}},
volume={{Second Edition}},
}

@article{fernique,
author={Fernique, Xavier},
title={{Integrability of Gaussian vectors}},
journal={{Comptes Rendus de l'Académie des Sciences, Série AB}},
volume={270},
pages={A1698 –A1699}
}

@article{BortoliConverge,
  title={Convergence of denoising diffusion models under the manifold hypothesis},
  author={De Bortoli, Valentin},
  journal={arXiv preprint arXiv:2208.05314},
  year={2022}
}

@inproceedings{ConvergeChen,
  title={Improved analysis of score-based generative modeling: User-friendly bounds under minimal smoothness assumptions},
  author={Chen, Hongrui and Lee, Holden and Lu, Jianfeng},
  booktitle={International Conference on Machine Learning},
  pages={4735--4763},
  year={2023},
  organization={PMLR}
}

@article{zuazua2005propagation,
  title={Propagation, observation, control and numerical approximation of waves},
  author={Zuazua, Enrique},
  journal={SIAM Review},
  volume={47},
  number={2},
  pages={197--243},
  year={2005}
}

@article{SongImproved,
  title={Improved techniques for training score-based generative models},
  author={Song, Yang and Ermon, Stefano},
  journal={Advances in neural information processing systems},
  volume={33},
  pages={12438--12448},
  year={2020}
}

@article{HyvarinenEstimation,
  title={Estimation of non-normalized statistical models by score matching.},
  author={Hyv{\"a}rinen, Aapo and Dayan, Peter},
  journal={Journal of Machine Learning Research},
  volume={6},
  number={4},
  year={2005}
}

@incollection{BookKloeden,
  title={{Stochastic Differential Equations}},
  author={Kloeden, Peter E},
  booktitle={International Encyclopedia of Statistical Science},
  pages={1520--1521},
  year={2011},
  publisher={Springer}
}

@book{RoydenReal,
  title={{Real Analysis}},
  author={Royden, Halsey Lawrence},
  edition={{Second Edition}},
  year={1968},
  publisher={The Macmillan Company, London}
}

@article{BookPavliotis,
  title={{Stochastic Processes and Applications}},
  author={Pavliotis, Grigorios A},
  journal={Texts in Applied Mathematics},
  volume={60},
  year={2014},
  publisher={Springer}
}

@book{BookOksendal,
  title={{Stochastic Differential Equations}},
  author={{\O}ksendal, Bernt and {\O}ksendal, Bernt},
  year={2003},
  publisher={Springer}
}

@article{HoDiffusion,
  title={{Denoising Diffusion Probabilistic Models}},
  author={Ho, Jonathan and Jain, Ajay and Abbeel, Pieter},
  journal={Advances in Neural Information Processing Systems},
  volume={33},
  pages={6840--6851},
  year={2020}
}

@article{large_behavior_models,
  title={{A Careful Examination of Large Behavior Models for Multitask Dexterous Manipulation}},
  author={Barreiros, Jose and Beaulieu, Andrew and Bhat, Aditya and Cory, Rick and Cousineau, Eric and Dai, Hongkai and Fang, Ching-Hsin and Hashimoto, Kunimatsu and Irshad, Muhammad Zubair and Itkina, Masha and others},
  journal={arXiv Preprint arXiv:2507.05331},
  year={2025}
}

@article{diffusion_visuomotor,
  title={{Diffusion Policy: Visuomotor Policy Learning via Action Diffusion}},
  author={Chi, Cheng and Xu, Zhenjia and Feng, Siyuan and Cousineau, Eric and Du, Yilun and Burchfiel, Benjamin and Tedrake, Russ and Song, Shuran},
  journal={The International Journal of Robotics Research},
  pages={02783649241273668},
  year={2023},
  publisher={SAGE Publications Sage UK: London, England}
}

@inproceedings{BehaviorCloningBagnell,
  title={{A Reduction of Imitation Learning and Structured Prediction to No-Regret Online Learning}},
  author={Ross, St{\'e}phane and Gordon, Geoffrey and Bagnell, Drew},
  booktitle={Proceedings of the Fourteenth International Conference on Artificial Intelligence and Statistics},
  pages={627--635},
  year={2011},
  organization={JMLR Workshop and Conference Proceedings}
}

@incollection{BillardImitation,
  title={{Imitation Learning in Robots}},
  author={Billard, Aude and Grollman, Daniel},
  booktitle={Encyclopedia of the Sciences of Learning},
  pages={1494--1496},
  year={2012},
  publisher={Springer}
}

@article{pi_nut_point_5,
  title={{$\pi_{0.5}$: A Vision-Language-Action Model with Open-World Generalization}},
  author={Intelligence, Physical and Black, Kevin and Brown, Noah and Darpinian, James and Dhabalia, Karan and Driess, Danny and Esmail, Adnan and Equi, Michael and Finn, Chelsea and Fusai, Niccolo and others},
  journal={arXiv Preprint arXiv:2504.16054},
  year={2025}
}

@article{StefanScoreMatching,
  title={{Generative Modeling by Estimating Gradients of the Data Distribution}},
  author={Song, Yang and Ermon, Stefano},
  journal={Advances in Neural Information Processing Systems},
  volume={32},
  year={2019}
}

@article{OpenVLA,
  title={{OpenVLA: An Open-Source Vision-Language-Action Model}},
  author={Kim, Moo Jin and Pertsch, Karl and Karamcheti, Siddharth and Xiao, Ted and Balakrishna, Ashwin and Nair, Suraj and Rafailov, Rafael and Foster, Ethan and Lam, Grace and Sanketi, Pannag and others},
  journal={arXiv Preprint arXiv:2406.09246},
  year={2024}
}

@inproceedings{SohlDeepUnsupervised,
  title={{Deep Unsupervised Learning Using Nonequilibrium Thermodynamics}},
  author={Sohl-Dickstein, Jascha and Weiss, Eric and Maheswaranathan, Niru and Ganguli, Surya},
  booktitle={International Conference on Machine Learning},
  pages={2256--2265},
  year={2015},
  organization={PMLR}
}

@book{abrahamowitz,
author ={Abramowitz, M. and Stegun, I. A.},
title={Handbook of Mathematical Functions.},
year={1965},
publisher={Dover, New York},
}

@article{infinitediffusion,
  title={{Infinite-Dimensional Diffusion Models}},
  author={Pidstrigach, Jakiw and Marzouk, Youssef and Reich, Sebastian and Wang, Sven},
  journal={arXiv Preprint arXiv:2302.10130},
  year={2023}
}

@article{DeepONet,
  title={{DeepONet: Learning Nonlinear Operators for Identifying Differential Equations Based on the Universal Approximation Theorem of Operators}},
  author={Lu, Lu and Jin, Pengzhan and Karniadakis, George Em},
  journal={arXiv Preprint arXiv:1910.03193},
  year={2019}
}

@article{StuartPDEs,
  title={{Inverse Problems: A Bayesian Perspective}},
  author={Stuart, Andrew M},
  journal={Acta Numerica},
  volume={19},
  pages={451--559},
  year={2010},
  publisher={Cambridge University Press}
}

@article{Hairer2009,
  title={{An Introduction to Stochastic PDEs}},
  author={Hairer, Martin},
  journal={arXiv preprint arXiv:0907.4178},
  year={2009}
}

@article{alpamayo,
  title={{Alpamayo-R1: Bridging Reasoning and Action Prediction for Generalizable Autonomous Driving in the Long Tail}},
  author={{Nvidia}},
  year={2025},
  url={https://d1qx31qr3h6wln.cloudfront.net/publications/Alpamayo-R1_1.pdf}
}

@article{DiffusionPrinciples,
  title={{The Principles of Diffusion Models}},
  author={Lai, Chieh-Hsin and Song, Yang and Kim, Dongjun and Mitsufuji, Yuki and Ermon, Stefano},
  journal={arXiv preprint arXiv:2510.21890},
  year={2025}
}

@manual{LevelSetPy,
  title = {{LevelSetPy: Hamilton-Jacobi Equations in Python}},
  author = {Molu, Lekan},
  year = {2024},
  organization = {Molux Labs},
  note = {\url{https://github.com/lekanmolu/LevelSetPy}},
  version = {1.2.0}
}

@book{GoldrattTheory,
  title = {{The Theory of Constraints}},
  author = {Goldratt, Eliyahu M.},
  year = {1984},
  publisher = {North River Press},
  address = {Great Barrington, MA}
}

@article{ConwipOrigin,
  title = {{The CONWIP Concept: A Pull System for Job Shops}},
  author = {Spearman, Mark L. and Zazanis, Michael A.},
  journal = {International Journal of Production Research},
  volume = {28},
  number = {9},
  pages = {1785--1794},
  year = {1990},
  doi = {10.1080/00207549008945535}
}

@article{LittleLaw,
  title = {{A Proof of the Queueing Formula $L = \lambda W$}},
  author = {Little, John D. C.},
  journal = {Operations Research},
  volume = {9},
  number = {3},
  pages = {383--387},
  year = {1961},
  doi = {10.1287/opre.9.3.383}
}

@article{HJSafety,
  title = {{Computing Certified Safety Envelopes for Autonomous Systems via Hamilton-Jacobi Reachability}},
  author = {Chen, Miriam and Mackin, Andrew and Liu, Jayson and Pavone, Marco},
  journal = {IEEE Transactions on Robotics},
  volume = {38},
  number = {2},
  pages = {1022--1039},
  year = {2022},
  doi = {10.1109/TRO.2021.3096782}
}

@manual{conwip_survey,
  title = {{CONWIP Systems: Principles and Implementation}},
  author = {Spearman, Mark L. and Zazanis, Michael A.},
  journal = {International Journal of Production Economics},
  volume = {210},
  pages = {1--15},
  year = {2019}
}

@article{lstm_conwip,
  title = {{LSTM-Based CONWIP Control for Manufacturing Flow}},
  author = {Chen, Wei and Liu, Yan and Zhang, Qi},
  journal = {IEEE Transactions on Automation Science and Engineering},
  volume = {18},
  number = {3},
  pages = {1456--1468},
  year = {2021}
}

@article{transformer_scheduling,
  title = {{Transformer-Based Job Shop Scheduling}},
  author = {Wang, Li and Zhao, Ming and Sun, Feng},
  journal = {Journal of Intelligent Manufacturing},
  volume = {32},
  number = {5},
  pages = {1023--1038},
  year = {2021}
}

@article{funspace_diff1,
  title = {{Functional Diffusion Models for Image Generation}},
  author = {Roger, Jean-Marc and Szekely, Gábor},
  journal = {International Conference on Machine Learning},
  pages = {4567--4580},
  year = {2022}
}

@article{funspace_diff2,
  title = {{Neural Operators for Functional Time Series Prediction}},
  author = {Kovačič, Žiga and Wan, Rose Yu},
  journal = {Advances in Neural Information Processing Systems},
  volume = {35},
  pages = {6789--6802},
  year = {2022}
}


\etocdepthtag.toc{appendix}

\newpage
\onecolumn

\clearpage
\phantomsection
\section*{Appendix Contents}
\addcontentsline{toc}{section}{Appendix Contents}

\vspace*{0.5em}

\begingroup
\etocsettagdepth{main}{none}
\etocsettagdepth{appendix}{subsubsection}

\etocsetstyle{section}
  {\vspace{0.8em}}
  {}
  {\noindent\textbf{\etocnumber\quad\etocname}\nobreak\dotfill\nobreak\textbf{\etocpage}\par\vspace{2pt}}
  {}

\etocsetstyle{subsection}
  {}
  {}
  {\noindent\hspace{1.8em}\etocnumber\enspace\etocname\nobreak\dotfill\nobreak\etocpage\par\vspace{1pt}}
  {}

\etocsetstyle{subsubsection}
  {}
  {}
  {\noindent\hspace{3.6em}\etocname\nobreak\dotfill\nobreak\etocpage\par}
  {}

\tableofcontents
\endgroup

\newpage
\appendix
\renewcommand{\thesection}{\Alph{section}}

\newpage

\section{A Catalog of Definitions, Lemmas, and Theorems}
\label{app:back}
This section introduces the background for our contributions. Let us first describe our notations.

\subsection{Notations}
We employ the Banach spaces: the $l^p$ of real sequences with finite norm $\inner{x_n}{x_n} := \left(\sum_{i=1}^{\infty}|x_n| ^p\right)^{1/p}$ for a $p \in [1, +\infty)$; and the bounded continuous functions space $\BC(\topospace)$ of the topological space $\topospace$. For infinite-dimensional spaces, we consider the $l^2$-separable Hilbert space  $\cH= L^2([0, T], \R^d)$. Radon Gaussian measures are denoted $\measure \sim \mc{N}(m, \cov)$ on
$\cH$, with mean $m$ and covariance operator $C_\measure$.  The expectation of a random variable $X$ is represented as $\E(X)$, while $\Tr(\cov)$ denotes the trace of the bounded linear operator $\cov:\cH\to\cH$. The stochastic process of interest is $\{\process_t\}_{t\ge 0}$; $\measureset$,
$\borel$, and $\measure$ denote its sample space, filtration, and probability measure, respectively. The Cameron-Martin space is denoted as $\hilbert_C$. For a vector $v$, $\|v\|_{\hilbert}, \, \|v\|_2, \,\|v\|_{\cH_c}$ denote its norms with respect to the Hilbert, Euclidean, and~\citep{CameronMartin} spaces, respectively.  Its time and spatial derivatives are written $v_t$, $\nabla_x v$. For a compact operator $T: \mathcal{H} \to \mathcal{H}$ with singular values ${\sigma_k}$, the Hilbert–Schmidt norm is $\|T\|^2_{\mathrm{HS}} := \sum_{k=1}^\infty \sigma_k^2$.


\renewcommand{\thesection}{\Alph{section}}
\renewcommand\theequation{\Alph{section}.\arabic{equation}}
\setcounter{equation}{0}


\subsection{Measure Theory on Polish Spaces}
%
We set our work in motion by briefly reviewing complete, separable, and metrizable spaces of probability measures on Polish spaces. 

A topological space $\topospace$ is separable if it contains a countable dense subset $(x_1, x_2, \cdots)  \in \topospace$ such that $\topospace = \text{clo}{(x_1, x_2, \cdots)}$, where $\text{clo}(x)$ denotes the closure of $x$. 
An \textbf{algebra} $\algebra$ of sets is a $\sigmalgebra$-algebra, or a \textbf{Borel field} $\borel$ if every union of a countable collection of sets in $\algebra$ is again in $\algebra$. Put differently, the $\sigmalgebra$-algebra $\borel$ is a family of subsets of a given set $\measureset$ which contains $\emptyset$ and is closed with respect to complements and countable unions. 

A \textbf{set function} $\mu$ assigns an extended real number to certain sets so that a \textbf{measurable space} $\measurablespace$ consists of the set $\measureset$  and a $\sigmalgebra$-algebra $\borel$ of subsets of $\measureset$. Throughout, we take the topological space $\topospace$ as $\topospace := \measurablespace$. A set $A \subset \measureset$ is said to be \textbf{measurable} with respect to $\borel$ if $A \in \borel$. 
The Banach space of bounded, real-valued, continuous functions on $\topospace$ is denoted $\mc{B}(\topospace, \|\cdot \|_{\mc{B}})$,  equipped with supremum norm $\|f\|_{\mc{B}} = \sup_{x \in \topospace} f(\process)$.  

A \textbf{measure} $\measure$ on the measurable space $\measurablespace$ is a nonnegative set function defined for all sets $\borel$ satisfying $\measure(\varnothing) = 0$ and $\measure(\cup_{i=1}^\infty E_i) = \sum_{i=1}^\infty \measure E_i$ for any sequence ${\mc{E}}=\{E_i\}_{i=1}^\infty$ of disjoint measurable sets. Since the measure runs to $\infty$ on ${\mc{E}}$, $\measure$ is said to be ``countably additive"; otherwise, $\measure$ is ``finitely additive" for disjoint sets $E_i \in \borel$. A set $\mc{E}$ is \textbf{finite} if $\measure(\measureset) < \infty$, and it is $\sigmalgebra$-\textbf{finite} if there exists a sequence $\{\measureset_n\} \in \borel$ such that $\measureset = \bigcup_{n=1}^\infty \measureset_n$ and $\measure (\measureset_n) < \infty$. The set $\mc{E}$ is of \textbf{finite measure} if $\mc{E} \in \borel$ and $\measure (\mc{E})  < \infty$. A set ${\mc{E}}$ is of $\sigmalgebra-$\textbf{finite measure} if ${\mc{E}}$ is the union of a countable collection of measurable sets of finite measure. The \textbf{measure space} $\measurespace$ is the measurable space $(\measureset, \borel)$ along with the measure $\measure$,  defined on $\borel$.  
A metric $l$ on $\topospace$ is consistent with $\borel$ if every set of the form $\{y \in \topospace \,|\, l(x,y) < c\}, \, x \in \topospace, \, c>0$, is in $\borel$, and every nonempty set in $\borel$ is the union of such sets. We say the space $(\measureset, \borel)$ is metrizable if such an $l$ exists. 

\begin{proposition}[\textbf{Measurable Functions}~\cite{RoydenReal}]
	Suppose that $f$ is an extended real-valued function with a measurable domain. Then, for each real number $\alpha$ the following statements are equivalent:
	\begin{enumerate}
		\item The set $\{\bm{x}: f(\bm{x}) > \alpha\}$ is measurable.
		\item The set $\{\bm{x}: f(\bm{x}) \ge \alpha\}$ is measurable.
		\item The set $\{\bm{x}: f(\bm{x}) < \alpha\}$ is measurable.
		\item The set $\{\bm{x}: f(\bm{x}) \le \alpha\}$ is measurable.
	\end{enumerate}
	These imply that for each extended real number $\alpha$ the set $\{\state: f(\state) = \alpha\}$ is measurable.
	\label{prop:measurable}
\end{proposition}
\begin{definition}[\textbf{Lebesgue measurable functions}]
	An extended real-valued function $f$ is \textbf{Lebesgue} measurable if its domain is measurable and if it satisfies one of the four statements above.
\end{definition}
%
%
%
%
%

\subsection{Convergence on Probability Measures}

A Borel measure $\measure$ is Radon if, for every Borel set $\borel$ and every $\varepsilon>0$, there exists a compact subset $K$ in $B$ such that $|\measure|(B\setminus(K) \le \varepsilon$~\cite[Section 1.2]{BogachevMeasure}. A measure $\nu$ is \textbf{absolutely continuous} with respect to the measure $\measure$ (denoted $\measure \ll \nu$) if $\nu (\topospace) = 0$ for each set $\topospace$ whereupon $\measure (\topospace) = 0$.

A \textbf{probability measure} $\measure$ on $X$ is a map $\measure: \borel(\measureset) \rightarrow [0,1]$ such that $\measure(\measureset) = 1$ and $\measure(\bigcup A_i) = \Sigma \measure(A_i)$ for any countable collection of mutually disjoint sets $A_i \in \borel(\measureset)$.

\begin{theorem}[\textbf{Radon-Nikodym Theorem}~\citep{RoydenReal}]
	\label{thm:radon-nikodym}
	Suppose that $(\measureset, \borel, \measure)$ is a $\sigma$-finite measure space. Suppose that $\nu$ is a measure defined on $\borel$ and is absolutely continuous with respect to $\measure$. Then there exists a nonnegative measurable function $f$ (the density at a given point) such that for each set $\mc{E} \in \borel$,
	\begin{align}
		\nu(\mc{E}) = \int_{\mc{E}} f d\measure.
	\end{align}
	In addition, the function $f$ is unique: if $g$ is any measurable function with this property, then $g = f$ a.e. $[\measure]$.
\end{theorem}
\begin{remark}[Radon-Nikodym derivative]
	The function $f$ is the Radon-Nikodym derivative of $\nu$ with respect to $\measure$, denoted $[{d\nu}/{d\measure}]$.
\end{remark}

\begin{definition}[\textbf{Pushforward Map}]
	For a continuous linear functional $f:\borel \rightarrow \reline$, the {pushforward map} on $\reline$  \ie $(f^\sharp\mu)(X) = \mu(f^{-1}(X))$ is a  measure on $\reline$ for every $f$. If for every zero-mean (\ie centered) $f^\sharp \mu$, $f$ is centered, then we say the \textbf{measure $\mu$ is centered}. 
\end{definition}

\begin{definition}[\textbf{Gaussian Measure}]
	The measure $\measure$  is \textit{Gaussian} if the pushforward $f^\sharp \measure$ is Gaussian for all linear functionals $f \in \process^\star$. 
	Such measure is \textit{non-degenerate} if it is strictly positive. 
\end{definition}

\begin{definition}[\textbf{Wiener process}, $\{W_t^{\cov}\}_{t\ge 0}$]
	A $\cov$-Wiener process (equivalently, a Wiener process
	with covariance operator $\cov$) is an $\mathcal{H}$-valued stochastic process $\{W_t^{\cov}\}_{t \ge 0}$ satisfying, 
	\begin{inparaenum}[(i)]
		\item $W_0^{\cov} = 0$ almost surely (a.s.);
		\item independent increments: $W_t^{\cov} - W_s^{\cov} \perp \mathcal{F}_s$ for all $0 \le s \le t$; and
		\item Gaussian increments: for every $h, g \in \mathcal{H}$,
		\begin{align}
			\mathbb{E}\!\left[\langle W_t^{\cov} - W_s^{\cov},\, h \rangle\, \langle W_t^{\cov} - W_s^{\cov},\, g \rangle\right]  = (t - s)\,\langle \cov h,\, g \rangle_{\mathcal{H}}.
		\end{align}
	\end{inparaenum}
	\label{def:WienerProcess}
\end{definition}
\begin{definition}[\textbf{Total variation convergence}]
	Let $(\measureset, \borel)$ be a fixed measurable space and let $\measure$ and $\nu$ be two (probability) measures defined on $(\measureset, \borel)$. The total variation convergence between two Gaussian measures $\measure$ and $\nu$ is defined as
	\begin{align}
		\label{eq:tv_dist}
		\|\measure - \nu\|_{TV} = \sup_{\|\phi\|_\infty \le 1} \bigg|\int \varphi(x) \measure(dx)  - \int \varphi(x) \nu(dx)\bigg|,
	\end{align}
	where $\|\phi\|_\infty$ denotes the supremum norm of $\varphi$.
\end{definition}
Equation \eqref{eq:tv_dist} is a useful metric for measuring the distance between two probability measures, otherwise defined for densities $D_\measure$ and $\mc{D}_\nu$ of the measures $\measure$ and $\nu$ as  
\begin{align}
	\|\measure - \nu\|_{TV} = \int_{\measureset} |\mc{D}_\measure(x)   - \mc{D}_\nu(x)|\pi(dx).
\end{align}
By the Radon-Nikodym \autoref{thm:radon-nikodym}, $\pi$ is a positive measurable function.

\begin{definition}[\textbf{The Cameron-Martin Space}]
	Let the adjoint of $\borel$ be denoted $\borel^\star$. Then  $\mu$ on $\borel$ is associated with the canonical Hilbert space $\hilbert_\mu \subset \borel$, called the \textit{Cameron-Martin space}~\citet{CameronMartin}, $E$ of measure $\measure$ — essentially a Hilbert space with the intersection of all linear spaces of full measure under $\mu$. It defines the set of directions whereby shifting $\measure$ to $\measure^\prime$ results in an absolutely continuous measure similar to $\measure$. Given a centered $\mu$, and for $f, g \in \borel^\star$ there exists a bilinear, positive definite operator of $\mu$, the map $\cov: \borel^\star \times \borel^\star \rightarrow \reline$ such that 
	\begin{align}
		f(\cov g) &= \cov(f, g) 
		:= \int_\borel f(x) g(x) \mu (dx),
	\end{align}
	for all $f \in \borel^\star,\, g \in \borel^\star_\measure$ where $\borel^\star_\measure$ is the closure of $\borel^\star$ in $L^2(\measure)$. By Fernique's theorem~\citet{fernique},~\cite[Th 3.36]{Hairer2009}, the operator norm, $C_\mu$ is bounded, \ie  there exists a constant $\|C_\mu\| < \infty$ such that $C_\mu(f, g) \le \|C_\mu \| \|f \| \|g \|$ due to the square integrable nature of the Gaussian measure (see ~\cite[Cor. 3.37]{Hairer2009}). Its self-adjoint, non-negative, and trace-class properties are established in Def.~\ref{def:trace_class}. 
\end{definition}

As it turns out, the covariance operator $C_\mu$ is not just bounded: if $\mc{B}(\topospace, \|\cdot \|_{\cH})\equiv \cH(\topospace, \|\cdot \|_{\cH})$, then the covariance operator $C_\mu(h, k)$ is of trace class. 
It is of interest to note that $C_\mu$ enjoys a representation in terms of the eigenvalues and the continuous representations of the eigenvectors as given in Mercer's Theorem \ref{thm:mercer}.  The eigenpairs $\{\lambda_k, e_k\}$ provide a coordinate system on
$\cH$ adapted to the Gaussian measure $\cN(0,C_\measure)$. These  are the backbone of the spectral analysis introduced shortly. %

\subsection{Spectral Analysis}
Consider $\mu$  on $\mathcal{H} = L^2(\process, \mu)$ with $\int_\mathcal{H} \|x\|^2_\mathcal{H}\,\mu(dx) < \infty.$ Define the \textbf{covariance operator} $\cov$ as
\begin{align}\langle \cov u,\, v \rangle_\mathcal{H}
	= \int_\mathcal{H} \langle x, u \rangle_\mathcal{H}\,\langle x, v \rangle_\mathcal{H}\;\mu(dx),
	\, \forall \,(u,v) \in \mathcal{H},
	\label{eq:cov_def}
\end{align}
with associated \textbf{covariance kernel} $k : \process \times \process \to \mathbb{R}^n$,
\begin{align}k(x,y) = \int_\mathcal{H} \langle z, x \rangle_\mathcal{H}\,\langle z, y \rangle_\mathcal{H}\;\mu(dz),
	\label{eq:cov_inner}
\end{align}
so that $(\cov f)(x) = \int_\process k(x,y)\,f(y)\,\mu(dy)$. We first restate Mercer's theorem.
\begin{theorem}{\textbf{Mercer's}, 1909; ~\cite[\S 98]{NagyRiez}}
	Suppose $k$ is continuous and
	symmetric, and $\mathcal{C}_\mu$ is positive, i.e.
	$\langle \mathcal{C}_\mu f, f \rangle_\mathcal{H} \ge 0$ for all $f \in \mathcal{H}$.
	Then,
	\begin{enumerate}
		\item \textbf{Spectral decomposition}. 
		There exists a countable orthonormal basis
		$\{e_k\}_{k=1}^\infty$ of $\mathcal{H}$ and strictly positive eigenvalues
		$\lambda_1 \ge \lambda_2 \ge \cdots > 0$, $\lambda_k \to 0$, 
		such that
		\begin{align}\mathcal{C}_\mu e_k = \lambda_k\, e_k\end{align}
		and $\mathcal{C}_\mu = \sum_{k=1}^\infty \lambda_k\,\langle\cdot, e_k\rangle_\mathcal{H}\, e_k$
		in the strong operator topology.
		\item  \textbf{Uniform kernel expansion}. 
		The series
		\begin{align}k(x, y) = \sum_{k=1}^\infty \lambda_k\, e_k(x)\, e_k(y)\end{align}
		converges absolutely uniformly on ${X} \times {X}$.
		\item \textbf{Diagonal and trace}. 
		On the diagonal, uniformly in $x$:
		$k(x,x) = \sum_{k=1}^\infty \lambda_k\,|e_k(x)|^2.$ Exchange of sum and integral is justified by Dini's theorem, since the partial sums of positive continuous functions converge uniformly to the continuous function $k(x,x)$. Integrating over ${X}$ yields the \textbf{trace identity},
		\begin{align}\operatorname{Tr}(\cov)
			= \int_\mathcal{X} k(x,x)\,\mu(dx)
			= \sum_{k=1}^\infty \lambda_k < \infty,\end{align}
		so that positivity and continuity of $k$ force $\mathcal{C}_\mu$ to be \textbf{trace-class}.
	\end{enumerate}
	\label{thm:mercer}
\end{theorem}

\subsection{Best $N$-term approximation}
\begin{corollary}[\textbf{Best $N$-term approximation}]
	\label{cor:best_n_term}
	For the uniformly convergent kernel  (see
	Theorem~\ref{thm:mercer})
	\begin{align}
		k(x,y) = \sum_{i=1}^\infty \lambda_i e_i(x)e_i(y),
	\end{align}
	the partial sum
	$k_N = \sum_{i=1}^N \lambda_i e_i \otimes e_i$ uniquely minimizes $\|k - k_N\|_{L^2}$
	with residual $\|k - k_N\|^2_{L^2} = \sum_{i>N}\lambda_i^2$, underpinning the
	effective rank $r_\mathrm{eff}(\mathcal{C}_\mu) = (\sum_k\lambda_k)^2/(\sum_k\lambda_k^2)$. 
\end{corollary}

\noindent\textbf{Proof of Corollary (Best $N$-term approximation).}
For coefficients $c_i \ge 0$, any rank-$N$ approximation
$\tilde{k}_N = \sum_{i=1}^N c_i f_i \otimes f_i$ for orthonormal $\{f_i\}$ satisfies
\begin{align}
	\|k - \tilde{k}_N\|^2_{L^2}
	= \|k\|^2_{L^2} - \sum_{i=1}^N c_i^2
	\;\le\; \sum_{i>N}\lambda_i^2,
\end{align}
with equality iff $\{(c_i,f_i)=(\lambda_i,e_i)\}_{i=1}^N$, by the Eckart-Young theorem applied to the self-adjoint Hilbert-Schmidt operator $\cov$ (with $\|\cov\|^2_{\mathrm{HS}}=\sum_k\lambda_k^2$).
Since $\lambda_1\ge\lambda_2\ge\cdots$, the Mercer partial sum $k_N=\sum_{i=1}^N\lambda_i e_i\otimes e_i$ is the unique minimizer, giving the effective rank
\begin{align}
	r_\mathrm{eff}(\cov)
	:= \frac{\bigl(\sum_k\lambda_k\bigr)^2}{\sum_k\lambda_k^2}
	= \frac{\operatorname{Tr}(\cov)^2}{\|\cov\|^2_{\mathrm{HS}}},
\end{align}
which equals $N$ for a uniform spectrum and is small when the spectrum decays rapidly. $\square$

The full hierarchy of norms and operators derived from the Mercer eigenpairs $\{(\lambda_k,e_k)\}$ is summarized in \autoref{tbl:norms_compare}.
The strict hierarchy trace-class $\Rightarrow$ Hilbert--Schmidt $\Rightarrow$ bounded follows from
$\sum_k\lambda_k < \infty \Rightarrow \sum_k\lambda_k^2 < \infty \Rightarrow \lambda_1 < \infty$.
The induced operator norm $\|\cov\|=\lambda_1$ is attained at $e_1$ by the spectral theorem; cf.\ the finite-dimensional case where the induced $2$-norm of a symmetric positive matrix equals its largest eigenvalue.

\begin{remark}[\textbf{Basis-free Trace-Class Operators}]
	For the non-negative symmetric operator
	$\cov : \mathcal{H} \to \mathcal{H}$, define
	\begin{align}\operatorname{Tr}(\cov) := \sum_{k=1}^\infty \langle \cov e_k,\, e_k \rangle_{\mathcal{H}}\end{align}
	for any complete orthonormal basis $\{e_k\}$.  By Lidskii's theorem, this sum is 	independent of the choice of basis, 
	and we say 	$\cov$ is \textbf{trace-class }if $\operatorname{Tr}(\cov) < \infty$.
	\label{def:trace_class}
\end{remark}

\subsection{Finite-Dimensional Diffusion Models in Coordinate Form}
\label{app:fd_diffusion}

In this section, we study the linear finite-dimensional SDE for ``diffusing" the initial measure $\measure_0$ to the infinite-dimensional generator formalism. Diffusion injects a stochastic Wiener process, $\{W_t^{\cov}\}_{t\ge 0}$ (see Def. \ref{def:WienerProcess}) into the forward Ornstein-Uhlenbeck (OU) process~\citep{BookPavliotis}, defined by the Itô forward SDE \eqref{eq:ou_sde}
such that the drift $-\tfrac{1}{2}X_t$ provides linear mean-reversion toward zero, and the noise $dW_t^{\cov}$ injects energy in every complete orthonormal basis direction $\{e_k\}_{k\ge 1}$ at rate $\sqrt{\lambda_k}$, where $\{\lambda_k\}_{k\ge 1}$ are the corresponding eigenvalues. 

Process  \eqref{eq:ou_sde} corrupts any clean action $a_0 \in \cH$  into pure colored noise $\cN(0,\cov)$ as $t\to\infty$ on discrete FD Euclidean spaces $\reline^d$~\citep{StefanScoreMatching}. \textit{As $d\rightarrow\infty$, the algorithm's stability deteriorates owing to the increasing refinement of discretization parameters.} {Function-space} diffusion-learned approximators are crucial for their extension to scientific problems. These may include spatio-temporal variables with  inverse PDE applications that recover material properties or dynamical system states from sparse or noisy measurements.

Consider the measure space $\measurespace$ and a stochastic process with events $\topospace_t=:\topospace$.  Let $\topospace_t$ evolve on the bounded and continuous domain $\BC(\R^n)$ that is associated with the transition probability density $p(y, t| x, s)$, where $0 \le s < t$. Let the measure $\measure$ be characterized by the stochastic transition matrix $P = \{p_{ij}\}$ such  that $\sum_i p_{ij}=1$.  For an observable $f \in \BC(\reline^n)$ at a time $t>0$, we are tasked with learning the transition probability density $p(y, t |  x, s)$ from which a smooth function  $u(x,s) = \expect[f(\topospace_t)\mid \topospace_s=x]$ emerges.   The Chapman-Kolmogorov poses the recovery of the value function operator $u(x,s)$  via the conditional expectation 
\begin{align}
	u(x,s) =  \expect[f(\topospace_t)\mid \topospace_s=x] := \int_{\R^n} f(y) p(y,t|x,s) dy.
	\label{eq:kceo}
\end{align}
Learning $u(x,s)$ provides an efficient compact representation of an high-dimensional information characterized by continuous-time, continuous sample paths; and  a dynamically consistent basis for  diffusion model evaluation as approximating paths $\topospace_t$.  Letting $p(y,t|x,s):= p$, the transition probability density evolves according to the \textit{Kolmogorov} equation in the backward variables $(x,s)$ as the final-value partial differential equation (pde)
%
\begin{align}
	-\dfrac{\partial u}{\partial s}(x,s) = b(x,s) \dfrac{\partial}{\partial s} u(x,s) + \frac{1}{2} \Sigma(x,s) \dfrac{\partial^2 }{\partial x^2} u(x,s), \, 
	u(t,x) = f(x) 
	\tag{Backward-Kolmogorov} 
	\label{eq:Backward-Kolmogorov}
\end{align}
%
where $b(x,s)$ is the drift of the diffusion process and $\Sigma(x,s)$ is the diffusion coefficient defined as 
\begin{align}
	b(x,s) &= \lim_{t\rightarrow s} \expect \left[\dfrac{\topospace_t - \topospace_s}{t-s}\mid \topospace_s = x\right] \text{ and } \Sigma(x,s) = \lim_{t\rightarrow s} \expect \left[\dfrac{\mid \topospace_t - \topospace_s\mid^2}{t-s}\mid \topospace_s = x\right].
	\label{eq:Diffusion-Coefficient}
\end{align}
Equation \ref{eq:Backward-Kolmogorov} models  \textit{backward Kolmogorov process}, which given a variable $x$ at a time $s$ in the past, its solution generates  $y$ at the current time $t$. 

A similar argument can be made for the forward variables $(y,t)$ if we consider the transition probability density function as the solution to the initial value pde problem 
%
\begin{align}
	-\dfrac{\partial p}{\partial t} = -\dfrac{\partial}{\partial y} b(t,y) \cdot p + \frac{1}{2} \dfrac{\partial^2}{\partial y^2} \left(\Sigma(t,y) \cdot p\right), \,
	p(y,s|x,s) = \delta(x-y). 
	\tag{Fokker-Planck}
	\label{eq:Forward-Kolmogorov}
\end{align}
%
Similarly, the \eqref{eq:Forward-Kolmogorov} equation describes the behavior of a stochastic diffusion process in the forward variables $(y,t)$. We refer readers to standard stochastic differential equations texts (such as ~\citep{BookOksendal, BookKloeden, BookPavliotis}) for the derivations above. For time-homogeneous systems, the time dependence in \eqref{eq:Backward-Kolmogorov} and \eqref{eq:Forward-Kolmogorov} vanish \ie $s=0$, and the resulting equations are similar to the forward and reverse processes employed in denoising probablistic models~\citep{HoDiffusion}; in these situations, the spatial partial derivatives are evaluated as Jacobian matrices. Since this is what we address here, we briefly introduce the time-homogeneous versions of the equations in what follows.

\subsection{Practical Covariance Functions}
\label{app:back::cov}
The covariance function is a crucial component of the  Gaussian process predictor for the data. Mercer's Theorem \ref{thm:mercer} treats the eigenfunction analysis of covariance functions, allowing us to express the covariance function under certain conditions in terms of its eigenfunctions and eigenvalues. A covariance function is invariant to input space transitions. Define the function $\tau = x - x^\prime$ and let $r = |\tau|$ for an isotropic covariance function. Also, define $\ell$ as thne characteristic length scale. In practice, we would choose $\cov$ from the  Mat\'ern class of covariance functions with kernel 
\begin{align}
	k_{Mat\acute{e}rn}(r) = \dfrac{2^{1-\nu}}{\Gamma(\nu)} \left(\dfrac{\sqrt{2 \nu}r}{\ell}\right)^\nu K_\nu\left(\dfrac{\sqrt{2 \nu} r}{\ell}\right),
\end{align}
where $\ell$ is the characteristic length-scale of the process --- defined as the amount of distance one has to move to see a marked increase in the function; $\nu > 0$ and $K_\nu>0$ is a modified Bessel function~\cite[Sec. 9.6]{abrahamowitz}. For a $d$-dimensional data,  this covariance function $k_{Mat\acute{e}rn}(r)$ has a spectral density
\begin{align}
	S(s)
	=
	\frac{2d\,\pi^{d/2}\,\Gamma\!\left(\nu+\frac{d}{2}\right)(2\nu)^\nu}
	{\Gamma(\nu)\,\ell^{2\nu}}
	\left(\frac{2\nu}{\ell^2}+4\pi^2 s^2\right)^{-\left(\nu+\frac{d}{2}\right)}, \nonumber
\end{align}
where $\Gamma(\cdot)$ is the gamma function~\cite[Sec. 6]{abrahamowitz}. Mat\'ern covariance functions become especially simple when $\nu$ is half-integer \ie $\nu = p + 1/2$, where $p$ is a non-negative integer. In this case the covariance function is a product
of an exponential and a polynomial of order $p$, the general expression can be
derived from ~\cite[eq. 10.2.15]{abrahamowitz}, giving
\begin{align}
	k_{\nu=p+1/2}(r) &= \exp\left(\tfrac{-\sqrt{2\nu r}}{\ell}\right) \dfrac{\Gamma(p+1)}{\Gamma(2p+1)} \cdot  \sum_{i=0}^p \dfrac{(p+1)!}{i!(p-i)!}\left(\dfrac{\sqrt{8 \nu r}}{\ell}\right)^{p-1}. 
\end{align}
\newpage
\section{Related Works and Distinctions} 
\label{sec:relwork}

\noindent\textbf{Finite-dimensional diffusion.}
DDPM~\citep{HoDiffusion} and score-based models~\citep{StefanScoreMatching} discretize actions in $\mathbb{R}^d$ before training, yielding convergence bounds that degrade with $d$. De Bortoli~\citep{BortoliConverge} and Chen et al.~\citep{ConvergeChen} proved score-matching loss scales as $O(\sqrt{d})$.~\citep{zuazua2005propagation} showed that discretization of controlled PDEs can introduce spurious high-frequency modes and degraded controllability, motivating caution when repeatedly discretizing stochastic control trajectories.

\noindent\textbf{Infinite-dimensional generative models.}
~\citep{infinitediffusion} proved dimension-independent convergence in Hilbert spaces and \citep{StuartPDEs} established a function-space Bayesian perspective on inverse PDEs that produces a full characterization of all possible solutions, and their
relative probabilities. Our novelties include: \begin{inparaenum}[(i)] \item adopting the Kolmogorov backward PDE  for adjoint guidance without sampling instability; \item grounding the objective in Cameron-Martin geometry, following the Radon--Nikodym \autoref{thm:radon-nikodym}; and \item unified validation across visuomotor, manufacturing, and safety-critical domains. \end{inparaenum} While prior infinite-dimensional diffusion formulations retain stochastic trajectory sampling, our perspective emphasizes the associated Kolmogorov operators as the primary representation of policy evolution.

\noindent\textbf{Function-space and weighted diffusion.}
Recent approaches~\citep{funspace_diff1, funspace_diff2} do not expose operator residuals that quantify dynamical inconsistency or policy infeasibility along generated trajectories. 
We provide the Kolmogorov residual for failure detection and bottleneck identification, unified validation across three domains, and certified safety guarantees.

\noindent\textbf{Robotic visuomotor control.}
Diffusion visuomotor policies~\citep{diffusion_visuomotor} and VLA models~\citep{OpenVLA,pi_nut_point_5} discretize actions, yielding high inter-step drift during denoising steps. Our ID formulation eliminates discretization artifacts by training in $\mathcal{H} = L^2([0,T], \mathbb{R}^{d_a})$ with three substitutions to standard DDPM. 

\noindent\textbf{Manufacturing flow control and supply chain optimization.}
Applications of diffusion policies to manufacturing systems have been limited, with most work focusing on discrete event simulation or traditional time series forecasting~\citep{conwip_survey}. Recent attempts to apply neural networks to production line control~\citep{lstm_conwip, transformer_scheduling} suffer from the same FD discretization issues that plague FD policies. Our work bridges this gap by modeling queue-length trajectories as function-valued processes in $L^2([0,T], \mathbb{R}^{n_{s}})$\footnote{For $n_s$ number of stations.}, preserving long-range temporal and inter-station correlations through the Cameron--Martin objective. We connect ID diffusion with HJ reachability for certified operating envelopes of inventories, work-in-process (WIP) management~\citep{GoldrattTheory}, and demand forecasting --- addressing the constraints/bottlenecks in stochastic manufacturing lines~\citep{LittleLaw}.

\noindent\textbf{Control applications.}
In autonomous control domains, diffusion policies find applications as action generators in vision-language-action (VLA) heads in autonomous driving~\citep{alpamayo}, whole-body control~\cite{beyondmimic}, and manipulation~\cite{pi_nut_point_5}, among others. However, these applications often neglect the theoretical guarantees provided by ID formulations. Our work connects ID diffusion with HJ reachability~\citep{HJSafety} to provide provably safe operating zones, addressing the ``safety-critical'' concerns of real-world deployment of VLAs. We further extend this to manufacturing flow control, where we demonstrate deadlock prevention and WIP throttling based on backward reachable sets of congestion states.

\noindent\textbf{Neural operators for function-space denoising.}
Fourier Neural Operators (FNO)~\citep{FNO} and DeepONet~\citep{DeepONet} learn mappings between ID function spaces without committing to a fixed discretization, making them natural candidates for the denoising network $\eta_\theta : \mathcal{H} \times [0,T] \to \mathcal{H}$ in our framework. In principle, operator-learning architectures may enable horizon-transfer behavior unavailable to fixed-grid denoisers, potentially narrowing the gap between continuous Kolmogorov formulations and their discrete implementations. 


\newpage
\section{Diffusion in Infinite Dimensions}
\label{app:innovs}

In this section, we present our infinite-dimensional diffusion policy framework. Our development is premised on measures that are approximated from within, i.e., Radon measures, and we leverage the fact that all Borel measures on complete separable metric spaces are Radon.

\subsection{The Infinite-Dimensional Diffusion Lifting}
\label{sec:infinite_dimensional_lifting}


\begin{table}[tb!]
	\centering
	\renewcommand{\arraystretch}{1.3}
	\begin{tabular}{|p{0.28\linewidth}|p{0.26\linewidth}|p{0.45\linewidth}|}
		\hline
		\textbf{Quantity} & \textbf{Spectral form} & \textbf{Meaning} \\
		\hline
		
		Action on $a \in \mathcal{H}$ &
		\[
		\cov
		=
		\sum_k \lambda_k
		\langle x,e_k\rangle_{\mathcal H}\, e_k
		\]
		&
		Decompose, scale, recompose
		\\
		\hline
		
		Operator norm &
		\[
		\|\cov\|
		=
		\lambda_1
		\]
		&
		Largest stretch on the unit ball in $\mathcal H$
		\\
		\hline
				
		Trace (trace-class condition) &
		\[
		\operatorname{Tr}(\cov)
		=
		\sum_k \lambda_k
		< \infty
		\]
		&
		Finite expected energy
		\\
		\hline
		
		Hilbert--Schmidt norm &
		\[
		\|\cov\|_{\mathrm{HS}}^2
		=
		\sum_k \lambda_k^2
		\]
		&
		Stronger than the trace-class
		\\
		\hline
		
		Square root &
		\[
		\cov^{1/2} e_k
		=
		\sqrt{\lambda_k}\, e_k
		\]
		&
		Positive functional calculus
		\\
		\hline
		
		Cameron--Martin norm &
		\[
		\|h\|_{\mathcal H_C}^2
		=
		\sum_k \frac{\hat h_k^2}{\lambda_k}
		\]
		&
		Anisotropic geometry induced by covariance
		\\
		\hline
		
	\end{tabular}
	\caption{Spectral identities associated with the covariance operator
		$\cov$.}
	\label{tbl:norms_compare}
\end{table}
\begin{proposition}[Trace-Class Covariance Operators]
	\label{app:back::trace_measure_theoretic}
	Let $\cov: \cH \to \cH$ be a positive, self-adjoint, trace-class operator on a separable Hilbert space $\cH$ with orthonormal eigenbasis $\{e_k\}$. Then $\operatorname{Tr}(\cov) = \sum_{k=1}^\infty \langle \cov e_k, e_k\rangle_\cH = \int_\cH \|x\|_\cH^2 \,\measure(dx)$ for any measure $\measure$ with covariance operator $\cov$, and paths of the $\cov$-Wiener process lie in $\cH$ almost surely.
\end{proposition}

\noindent\textbf{Proof.}
From \eqref{eq:cov_def},
\begin{align}
	\langle \cov e_k,\, e_k\rangle_\cH
	= \int_\cH \langle x,\, e_k\rangle^2\,\measure(dx).
\end{align}
Summing over the orthonormal basis and applying Parseval's identity
$(\|x\|^2_\cH = \sum_k\langle x,e_k\rangle^2)$,
\begin{align}
	\operatorname{Tr}(\cov)
	= \sum_{k=1}^\infty \int_\cH \langle x,e_k\rangle^2\,\measure(dx)
	= \int_\cH \|x\|^2_\cH\,\measure(dx),
\end{align}
where Tonelli's theorem (all terms non-negative)~\cite[pp.~270]{RoydenReal} justifies exchanging sum and integral.
Hence $\operatorname{Tr}(\cov)<\infty$ ensures paths of the $\cov$-Wiener process lie in $\cH$ a.s. $\square$

\begin{corollary}[\textbf{RKHS interpretation}]
	\label{rem:rkhs}
	The Cameron-Martin space $\cH_C=\cov^{1/2}(\cH)$ is the RKHS of $k$:
	for every $x\in\mathcal{X}$, $k(\cdot,x)\in \cH_C$ with reproducing property
	\begin{align}
		\langle f,\,k(\cdot,x)\rangle_{\cH_C} = f(x), \quad \forall\, f\in \cH_C,\, x\in\mathcal{X}.
	\end{align}
	In the Mercer eigenbasis, $k(x,y)=\sum_{k}\lambda_k e_k(x)e_k(y)$ is the
	reproducing kernel representation.
	The canonical inclusion $\mc{I}: \cH_C \hookrightarrow\cH$ is Hilbert-Schmidt norm,
	\begin{align}
		\|\mc{I}\|^2_{\mathrm{HS}}
		= \sum_{k=1}^\infty \|\mc{I}e_k\|^2_\cH
		= \sum_{k=1}^\infty \lambda_k
		= \operatorname{Tr}(\cov) < \infty.
	\end{align}
	Thus, the trace-class condition on $\cov$ is equivalent to $\cH_C\hookrightarrow\cH$ being a Hilbert-Schmidt embedding.
\end{corollary}

\begin{corollary}[\textbf{Measure-theoretic characterization.}] 
	\label{rem:measure_theoretic}
	Let $\mu = \mathcal{N}(0, \cov)$ be the Gaussian
	measure on $\mathcal{H}$ with covariance $\cov$.  By definition of $\cov$ as the covariance
	operator of $\mu$, we have
	\begin{align}\langle \cov e_k,\, e_k \rangle_{\mathcal{H}} = \int_{\mathcal{H}} \langle x,\, e_k \rangle^2 \,\mu(dx).\end{align}
	Summing over the basis and applying Parseval's identity
	$\|x\|_\mathcal{H}^2 = \sum_k \langle x, e_k\rangle^2$ together with Tonelli's theorem
	(non-negative terms justify interchange of sum and integral),
	\begin{align}\operatorname{Tr}[\cov]
	= \sum_{k=1}^\infty \int_{\mathcal{H}} \langle x,\, e_k \rangle^2 \,\mu(dx)
	= \int_{\mathcal{H}} \|x\|_{\mathcal{H}}^2 \,\mu(dx).\end{align}
	Hence $\operatorname{Tr}[\cov] < \infty$ is equivalent to $\mu$ having finite second 	moment in $\mathcal{H}$ i.e., typical samples have finite $\mathcal{H}$-norm. Choosing $\{e_k\}$ to be the Mercer eigenbasis gives the three-way equivalence,
	\begin{align}\operatorname{Tr}[\cov]
	\;=\; \int_{\mathcal{H}} \|x\|_{\mathcal{H}}^2\,\mu(dx)
	\;=\; \sum_{k=1}^\infty \langle \cov e_k,\, e_k\rangle_{\mathcal{H}}
	\;\triangleq\; \sum_{k=1}^\infty \lambda_k \;<\; \infty.\end{align}
\end{corollary}

\begin{remark}
		The measure-theoretic form is the most fundamental; the basis-free operator form shall
	what appears in the BKE diffusion term (to be introduced shortly); the eigenvalue sum is the most computationally
	explicit.
\end{remark}
\subsection{Dimension Independent Convergence}
\begin{proof}[Proof of Dimension-Independent Convergence Theorem \ref{thm:dim_indep}]
	\label{rem:dim_ind_converge}

The proof establishes convergence of the learned policy $\mu_\theta$ to the data distribution $\mu_{\text{data}}$ in total variation distance by decomposing the error into learning (data fitting) and approximation (BKE numerical integration) components. The key insight is that all constants depend only on spectral properties of the data covariance $\cov$, not on action dimension.

\noindent\textbf{Setup: Two Markov diffusions.}
Both $\mu_{\text{data}}$ and $\mu_\theta$ are generated by the OU process  \eqref{eq:ou_sde} with  denoising drifts,
\begin{align}
dX_s^{\text{data}} &= -\tfrac{1}{2}X_s\,ds + \eta^*(X_s,s)\,ds + dW_s^{\cov}, \label{eq:data_diffusion}\\
dX_s^\theta &= -\tfrac{1}{2}X_s\,ds + \eta_\theta(X_s,s)\,ds + dW_s^{\cov}, \label{eq:learned_diffusion}
\end{align}
where $\eta^*(x,s)$ is the optimal denoising direction determined by the value function $u^*$ via \eqref{eq:bke_method} and $\eta_\theta(x,s)$ is the neural approximation learned via the Cameron-Martin loss. Both processes share the same Wiener noise, $W_s^{\cov}$, and drift terms, $-1/2 X_s ds$; they differ only in the denoising correction.

\noindent\textbf{Step 1: Girsanov's theorem and KL divergence.}
By Girsanov's theorem for infinite-dimensional diffusions~\citep[\S 8.6]{BookOksendal}, the Radon-Nikodym derivative (Theorem~\ref{thm:radon-nikodym}) of $\mu_\theta$ with respect to $\mu_{\text{data}}$ is
\begin{align}
\frac{d\mu_\theta}{d\mu_{\text{data}}}(X^{\text{data}}) = \exp\left(\int_0^T \left\langle \eta_\theta - \eta^*, dW_s^{\cov}\right\rangle_{\mathcal{H}}  - \tfrac{1}{2}\int_0^T \left\|\eta_\theta(X_s,s) - \eta^*(X_s,s)\right\|_{\mathcal{H}_\mathcal{C}}^2\,ds\right),
\end{align}
with KL divergence 
\begin{align}
\operatorname{KL}(\mu_{\text{data}} \| \mu_\theta) &= \mathbb{E}_{\mu_{\text{data}}}\left[\log\frac{d\mu_{\text{data}}}{d\mu_\theta}\right] = \mathbb{E}_{\mu_{\text{data}}}\left[-\int_0^T \left\langle \eta_\theta - \eta^*, dW_s^{\cov}\right\rangle_{\mathcal{H}} + \tfrac{1}{2}\int_0^T \left\|\eta_\theta - \eta^*\right\|_{\mathcal{H}_\mathcal{C}}^2\,ds\right].
\end{align}
The stochastic integral $\int_0^T \left\langle \eta_\theta - \eta^*, dW_s^{\cov}\right\rangle_{\mathcal{H}}$ is a martingale with respect to the filtration generated by $W^{\cov}$ (the Itô isometry holds for trace-class noise by Proposition~\ref{app:back::trace_measure_theoretic}). Its expectation vanishes, leaving
\begin{align}
\operatorname{KL}(\mu_{\text{data}} \| \mu_\theta) &= \tfrac{1}{2}\int_0^T \mathbb{E}_{X_s \sim \mu_{\text{data}}}\left[\left\|\eta_\theta(X_s,s) - \eta^*(X_s,s)\right\|_{\mathcal{H}_\mathcal{C}}^2\right]\,ds.
\label{eq:kl_score_error}
\end{align}

\noindent\textbf{Step 2: Cameron-Martin loss controls score error.}
The learned denoiser $\eta_\theta$ is trained to minimize the Cameron-Martin loss
\begin{align}
\mathcal{L}_{\mathrm{CM}}(\theta) = \mathbb{E}_{s,X_s,\eta}\left[\left\|\cov^{-1/2}(\eta_\theta(X_s,s) - \eta)\right\|_{\mathcal{H}}^2\right], \quad \eta \sim \mathcal{N}(0,\cov),
\end{align}
which is precisely designed to measure error in the Cameron-Martin norm, $\|v\|_{\mathcal{H}_\mathcal{C}}^2 = \|\cov^{-1/2}v\|_{\mathcal{H}}^2$ (see Corollary~\ref{rem:measure_theoretic}). At optimality, the loss bounds the squared Cameron-Martin norm of the score error,
\begin{align}
\int_0^T \mathbb{E}_{X_s}\left[\left\|\eta_\theta(X_s,s) - \eta^*(X_s,s)\right\|_{\mathcal{H}_\mathcal{C}}^2\right]\,ds \;\lesssim\; \mathcal{L}_{\mathrm{CM}}(\theta).
\label{eq:loss_error_bound}
\end{align}
The relationship follows from the orthogonality of residuals at optimality and the equivalence $\|\cov^{-1/2}v\|_{\mathcal{H}} = \|v\|_{\mathcal{H}_\mathcal{C}}$. Combining equations \eqref{eq:kl_score_error} and \eqref{eq:loss_error_bound}, we must have
\begin{align}
\operatorname{KL}(\mu_{\text{data}} \| \mu_\theta) \;\lesssim\; \mathcal{L}_{\mathrm{CM}}(\theta).
\label{eq:kl_loss}
\end{align}

\noindent\textbf{Step 3: Pinsker's inequality and learning error.}
By Pinsker's inequality (see e.g.~\cite{Hairer2009}), for any two probability measures on a separable metric space,
\begin{align}
\|\mu_\theta - \mu_{\text{data}}\|_{\text{TV}} \leq \sqrt{\tfrac{1}{2}\operatorname{KL}(\mu_{\text{data}} \| \mu_\theta)}.
\label{eq:pinsker}
\end{align}
Substituting equation \eqref{eq:kl_loss},
\begin{align}
\|\mu_\theta - \mu_{\text{data}}\|_{\text{TV}} \leq C_1 \sqrt{\mathcal{L}_{\mathrm{CM}}(\theta)},
\label{eq:learning_error}
\end{align}
where $C_1 = \sqrt{C'/2}$ is a universal constant depending only on the Girsanov and Pinsker coefficients.

\noindent\textbf{Step 4: Tail error from BKE approximation.}
In practice, the value function $u^*$ solving the backward Kolmogorov equation (equation \eqref{eq:bke_method}) is approximated numerically via Picard iteration or finite-difference schemes. The error in this numerical integration contributes an additional tail term. By the contraction property of the OU semigroup (the eigenvalues of the OU generator $\mathcal{L}$ are $-\lambda_k/2$ by spectral theory applied to the Mercer decomposition~\cite{BookPavliotis}), the Picard iterates converge geometrically with rate $e^{-T/2}$:
\begin{align}
\left\|u_{\text{num}}(x,0) - u^*(x,0)\right\|_{\mathcal{H}} \leq C_2 e^{-T/2},
\label{eq:bke_tail}
\end{align}
where $C_2$ depends on the terminal condition $f$ and properties of the OU operator but \textbf{not} on action dimension or discretization resolution.

\noindent\textbf{Step 5: Dimension-independence.}
The key observation is that all constants in equations \eqref{eq:pinsker}, \eqref{eq:learning_error}, and \eqref{eq:bke_tail} depend only on:
\begin{enumerate}
	\item Universal constants (Girsanov factor 1/2, Pinsker factor $\sqrt{1/2}$),
	\item The trace of the covariance operator, $\operatorname{Tr}(\cov) = \sum_{k=1}^\infty \lambda_k$ (Corollary~\ref{rem:measure_theoretic}),
	\item Properties of the OU semigroup spectrum, which is $\{-\lambda_k/2\}_{k=1}^\infty$.
\end{enumerate}
By Mercer's theorem (Theorem~\ref{thm:mercer}), $\operatorname{Tr}(\cov)$ is a scalar property of the covariance kernel, determined by its smoothness (e.g., for the Mat\'ern-3/2 kernel, $\lambda_k \sim k^{-3}$), \textbf{not} by the action dimension $d_a$ or discretization size. The Cameron-Martin norm $\|\cdot\|_{\mathcal{H}_\mathcal{C}}$ operates on function space without reference to dimension. The OU contraction rate $e^{-T/2}$ is uniform across all Mercer modes.

\noindent\textbf{Final bound.}
Combining the learning error (equation \eqref{eq:learning_error}) and the tail error (equation \eqref{eq:bke_tail}),
\begin{align}
\|\mu_\theta - \mu_{\text{data}}\|_{\text{TV}} &\leq C_1\sqrt{\mathcal{L}_{\mathrm{CM}}(\theta)} + C_2 e^{-T/2},
\label{eq:final_bound}
\end{align}
where the constants $C_1, C_2 > 0$ are determined by universal properties of Girsanov/Pinsker, the spectral trace $\operatorname{Tr}(\cov)$, and the OU operator contraction. Neither depends on the discretization resolution, the action dimension $d_a$, or the planning horizon $T$ (except exponentially in the tail term). This proves dimension-independent convergence and establishes the superiority of the infinite-dimensional formulation over finite-dimensional DDPM, which suffers an $O(\sqrt{d_a})$ degradation in its convergence rate.

\end{proof}

\begin{figure}[tb!]
	\centering
	\includegraphics[width=\columnwidth]{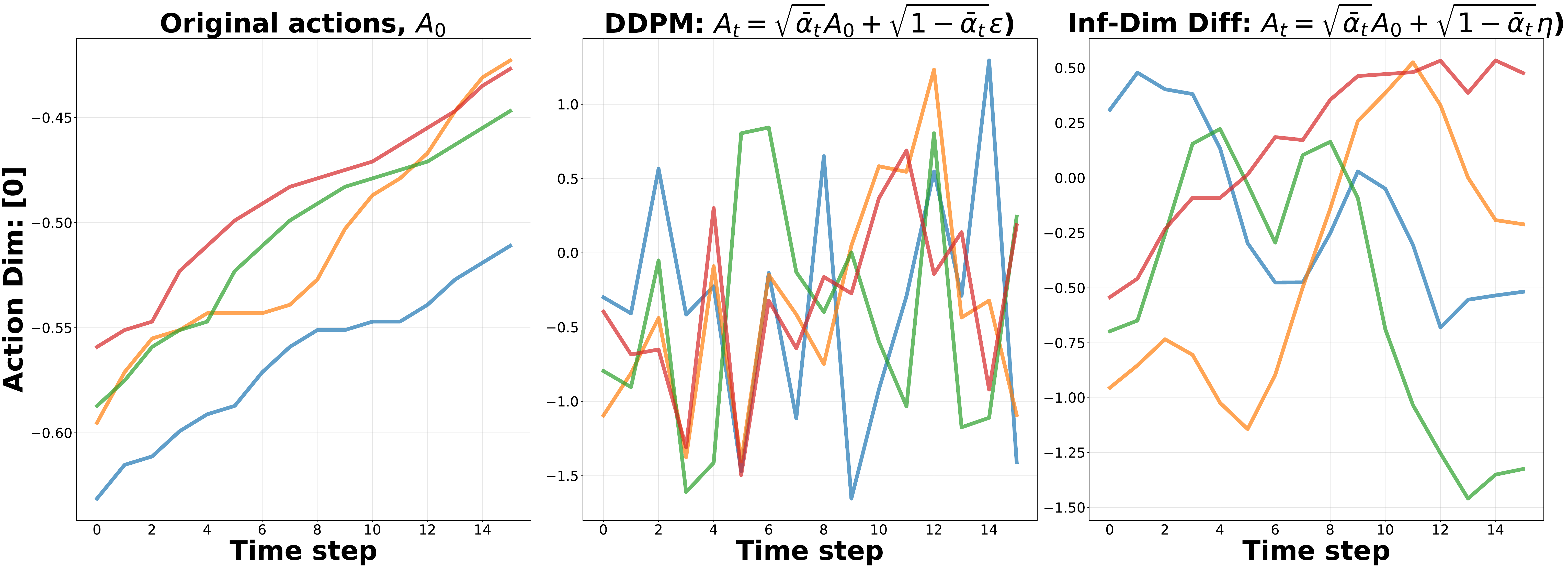}
	\caption{Noise comparison of the forward OU process, showing the variation on action distribution (left) when it is corrupted by white noise (middle) and with the Mat\'ern-3/2 noise.}
	\label{fig:ou_noise_compare}
\end{figure}

\subsection{Score Function in Infinite Dimensions}
\label{subsubsec:score_function}
A fundamental challenge in infinite-dimensional probability theory is that Gaussian measures on infinite-dimensional spaces have no Lebesgue density, so the conventional score $\nabla_x \log p_t(x)$ is ill-defined. We overcome this via the conditional expectation formulation.

\begin{definition}[Score Function via Conditional Expectation]
	Let $u(x,s) = \mathbb{E}[f(X_t) \mid X_s = x]$ for $f \in \mathcal{H}$. Then $u$ satisfies the backward Kolmogorov equation \eqref{eq:bke_method}. 
	
	The derivation of \eqref{eq:bke_method} proceeds thus. For a general $\mathcal{H}$-valued Itô process with
	\begin{align}dX_t = b(X_t)\,dt + dW_t^C,\end{align}
	the infinitesimal generator $\mathcal{L}$ encodes the instantaneous rate of change of expectations.  Acting on a smooth (Fr\'echet-differentiable) test function
	$\varphi : \mathcal{H} \to \mathbb{R}$, it reads
	\begin{align}\mathcal{L}\varphi(x)
		= \underbrace{\langle b(x),\, \nabla_x \varphi(x) \rangle_{\mathcal{H}}}_{\text{drift part}}
		+ \underbrace{\tfrac{1}{2}\operatorname{Tr}\!\left[C \cdot D^2\varphi(x)\right]}_{\text{diffusion part}}.\end{align}
	Here $\nabla_x \varphi(x) \in \mathcal{H}$ is the Fréchet gradient and
	$D^2\varphi(x) : \mathcal{H} \to \mathcal{H}$ is the Hessian (second Fréchet
	derivative).  The formula is the infinite-dimensional analogue of the finite-dimensional
	Itô generator, with $\cov$ playing the role of $\sigma\sigma^\top$. 	The score function is then defined as,
	\begin{align}
		\nabla_x \log p_s(x) = \cov^{-1} \nabla_x u(x,s).
		\label{eq:score_function_method}
	\end{align}
	This formulation replaces score-matching in $\mathcal{H}$ with a deterministic boundary-value problem that yields adjoint guidance without Monte Carlo sampling instability.
\end{definition}

\subsection{Cameron-Martin Loss and Convergence}

\noindent\textbf{Spectral decoupling of the forward process.}
The infinite-dimensional structure admits the simplification \ie, projecting the OU process \eqref{eq:ou_sde} onto the Mercer eigenbasis $\{e_k\}$, the dynamics decouple into independent scalar OU processes, one per mode. For each $k = 1, 2, \ldots$, define $a_k(t) := \langle X_t, e_k\rangle_\mathcal{H}$. Then
\begin{align}
	\frac{da_k}{dt} = -\tfrac{1}{2}a_k \, dt + \sqrt{\lambda_k}\, d\beta_k(t),
\end{align}
where $\{\beta_k\}$ are independent standard Brownian motions. Each mode decays at an exponential rate $1/2$ and is forced by colored noise scaled by $\sqrt{\lambda_k}$. Consequently, the stationary variance of mode $k$ is $\operatorname{Var}(a_k(\infty)) = \lambda_k$, consistent with the covariance structure. This spectral separation is the foundation of the dimension-independence claim: we do not discretize the space $\mathcal{H}$ into a finite grid; instead, we leverage the spectral structure to decouple infinite-dimensional coupling.

\noindent\textbf{Conditional score structure and insufficiency of naive loss.}
The conditional score $\nabla_{x} \log p_s(x)$ in the Mercer eigenbasis decomposes mode-wise. Solving the scalar OU SDE $da_k = -\tfrac{1}{2}a_k \, dt + \sqrt{\lambda_k}\,d\beta_k$ from $a_k(0)$ to time $s$ yields the marginal transition (via Itô's formula)
\begin{align}
	a_k(s) \mid a_k(0) \sim \mathcal{N}\left(e^{-s/2}a_k(0),\, \lambda_k(1 - e^{-s})\right),
\end{align}
where the mean decays exponentially and the variance grows from $0$ to $\lambda_k$. The conditional score for mode $k$ is then
\begin{align}
	\frac{\partial \log p_s}{\partial a_k}\bigg|_{a_k} = -\frac{e^{-s/2}a_k(0) - a_k(s)}{\lambda_k(1-e^{-s})}.
	\label{eq:score_mode_k}
\end{align}
A naive squared loss $\mathcal{L}_{\text{plain}}(\theta) = \mathbb{E}\left[\sum_k (\eta_{\theta,k} - \eta_k)^2\right]$, where $\eta_k := \sqrt{\lambda_k(1-e^{-s})}\,\xi_k$ is the colored noise in mode $k$ and $\xi_k \sim \mathcal{N}(0,1)$, would weight all modes equally. However, high-energy modes (large $\lambda_k$) have smoother score surfaces (gradients are divided by $\lambda_k$), while low-energy modes (small $\lambda_k$) concentrate density sharply. Equally weighting the prediction error across all modes ignores this asymmetry: a small error in a high-energy mode contributes little to the KL divergence, while an equally-sized error in a low-energy mode is catastrophic. The infinite-dimensional setting makes this pathology severe because the spectrum is unbounded below ($\lambda_k \to 0$).

\noindent\textbf{Cameron-Martin loss as measure-theoretic weighting.}
To correct for this imbalance, we weight the loss by the inverse covariance: the \textbf{Cameron-Martin loss} is
\begin{align}
	\mathcal{L}_{\mathrm{CM}}(\theta) = \mathbb{E}\!\left[\left\|\cov^{-1/2}(\eta_\theta - \eta)\right\|_{\mathcal{H}}^2\right].
	\label{eq:cm_loss_app}
\end{align}
%
Since $\cov e_k = \lambda_k e_k$, we have $\cov^{-1/2} e_k = \lambda_k^{-1/2} e_k$. Applying this to the expansion $\eta_\theta - \eta = \sum_k (\eta_{\theta,k} - \eta_k) e_k$ yields
\begin{align}
	\cov^{-1/2}(\eta_\theta - \eta) = \sum_{k=1}^\infty \lambda_k^{-1/2}(\eta_{\theta,k} - \eta_k) e_k.
\end{align}
Thus, \eqref{eq:cm_loss_app} in spectral form is
\begin{align}
	\mathcal{L}_{\mathrm{CM}}(\theta) = \mathbb{E}\!\left[\sum_{k=1}^\infty \frac{|\eta_{\theta,k} - \eta_k|^2}{\lambda_k}\right].
	\label{eq:cm_loss_spectral}
\end{align}
\begin{tcolorbox}[title=\textbf{Precision-Weighted Loss}, colback=gray!6, colframe=gray!50, fonttitle=\small, fontlower=\small]
	\small
	This \emph{precision weighting} penalizes errors in low-energy modes ($\lambda_k \to 0$) where density concentrates. it is measure-theoretically natural and dimension-independent since it depends only on the spectrum, rather than on $d_a$ or discretization.
	
\end{tcolorbox}

\subsection{Derivation of the Kolmogorov Residual Physics-Aware Diagnostic}
\label{app:residual_derivation}

The Kolmogorov residual \eqref{eq:kolmogorov_residual} in \S\ref{sec:kolmogorov_residual} emerges directly from the backward Kolmogorov equation by converting the PDE constraint into a measurable quantity.

\noindent\textbf{Derivation.}
The backward Kolmogorov equation (equation \eqref{eq:bke_method}) states that the true value function $u^*(x,s)$ satisfies
\begin{align}
-\frac{\partial u^*}{\partial s}(x,s) = \left\langle -\frac{1}{2}x,\, \nabla_x u^*(x,s) \right\rangle_{\mathcal{H}} + \frac{1}{2}\operatorname{Tr}\!\left[\cov \cdot \nabla_x^2 u^*(x,s)\right].
\label{eq:bke_form}
\end{align}
Rearranging by moving all terms to the left-hand side,
\begin{align}
\frac{\partial u^*}{\partial s}(x,s) + \left\langle -\frac{1}{2}x,\, \nabla_x u^*(x,s) \right\rangle_{\mathcal{H}} + \frac{1}{2}\operatorname{Tr}\!\left[\cov \cdot \nabla_x^2 u^*(x,s)\right] = 0.
\label{eq:bke_rearranged}
\end{align}
This identity holds \textit{exactly} for the true solution $u^*$. For a learned approximation $\hat{u}(x,s)$ parameterized by a neural network with weights $\theta$, we define the residual as the norm of the left-hand side of equation \eqref{eq:bke_rearranged}:
\begin{align}
\mathcal{R}(\hat{u}) &:= \left\|\frac{\partial \hat{u}}{\partial s} + \left\langle -\frac{1}{2}x,\, \nabla_x \hat{u} \right\rangle_{\mathcal{H}} + \frac{1}{2}\operatorname{Tr}\!\left[\cov \cdot \nabla^2 \hat{u}\right]\right\|_{\mathcal{H}}.
\end{align}
By construction, $\mathcal{R}(\hat{u}) = 0$ if and only if $\hat{u}$ satisfies the BKE exactly.

\noindent\textbf{Computational aspects.}
All three terms in the residual come from automatic differentiation of $\hat{u}$:
\begin{inparaenum}[(i)]
	\item $\frac{\partial \hat{u}}{\partial s}$ is the backward gradient w.r.t. the time input $s$,
	\item $\left\langle -\frac{1}{2}x,\, \nabla_x \hat{u} \right\rangle_{\mathcal{H}}$ is the directional derivative w.r.t. the state input $x$, and
	\item $\frac{1}{2}\operatorname{Tr}[\cov \cdot \nabla^2 \hat{u}]$ is a weighted trace of the Hessian, computed via the eigenvalues and eigenvectors of the covariance operator $\cov$ (via Mercer's theorem, Theorem~\ref{thm:mercer}).
\end{inparaenum}
Given the state $(x,s)$ and the network weights $\theta$, the residual can be evaluated in a single forward-backward pass with no environment interaction or data collection.

\noindent\textbf{Interpretation as unbiased failure detection.}
The residual is a pointwise, deterministic PDE constraint violation—not a statistical estimate. Unlike the training loss $\mathcal{L}_{\mathrm{CM}}(\theta)$, which averages over batches and measures data-fitting fidelity, the residual  interrogates  the learned function, evaluating whether it satisfies the BKE  point-wise. High residuals signal BKE violation of the conditional expectation structure it encodes.

\subsection{Why Colored Noise?}
Sampling discretizes the reverse SDE from $t = T$ to $t = 0$:
$$X_{t - \Delta t}
= \mu_\theta(X_t, t) + \sqrt{\tilde{\beta}_t}\,\eta_t,
\qquad \eta_t \sim \mathcal{N}(0, C),$$
where $\mu_\theta$ is the posterior mean from $\eta_\theta$ and $\tilde{\beta}_t$ is
the posterior variance schedule.

The forward noising injects $\cov$-colored increments $dW_t^{\cov}$ at every step; the reverse SDE is also driven by $dW_t^{\cov}$.  Using white noise $\mathcal{N}(0,I)$ instead would inject energy \textit{outside} the Cameron-Martin space
$E=\cov^{1/2}(\mathcal{H})$.  By the Cameron-Martin theorem~\citep{CameronMartin}, any perturbation
outside $E$ produces a measure mutually singular with $\mathcal{N}(0,\cov)$, meaning
the reverse trajectory immediately leaves the support of the data distribution. In spectral terms, white noise injects equal energy in every mode $e_k$; the data
distribution concentrates on low-$k$ modes (large $\lambda_k$); colored noise
$\mathcal{N}(0,\cov)$ injects energy proportional to $\lambda_k$, keeping the
trajectory inside $E$ at every step.  The replacement
$\mathcal{N}(0,I) \to \mathcal{N}(0,\cov)$ is therefore not cosmetic. It is a
measure-theoretic necessity.


\newpage
\section{Kolmogorov Generative Modeling in Infinite Dimensions}
\label{app:algorithm}

The infinite-dimensional framework is instantiated via discrete approximation on $[0, T]$ with $t_p$ timesteps. Algorithm~\ref{alg:infdiff} shows the complete pipeline with the three key substitutions highlighted.

\begin{algorithm}[t]
\caption{Infinite-Dimensional Diffusion Policy (InfDiff)}
\label{alg:infdiff}
\begin{algorithmic}[1]
\State \textbf{Input:} Dataset $\{a_0^{(i)}\}_{i=1}^{N}$; planning horizon $T$; timesteps $t_p$; Matérn kernel parameters $(\ell, \sigma^2, \nu)$
\State \textbf{Preprocessing:} Compute Gram matrix $K \in \mathbb{R}^{t_p \times t_p}$ with $K_{ij} = k(s_i, s_j)$
\State \quad Compute Cholesky factor $L = \mathrm{chol}(K)$
\State \textbf{\underline{TRAINING LOOP}}
\For{epoch $= 1, \ldots, N_{\mathrm{epochs}}$}
	\For{batch $\{a_0^{(b_i)}\}_{i=1}^{B}$}
		\State \textit{Substitution I:} Sample colored noise $\eta^{(b)} \sim \mathcal{N}(0, K)$ via $\eta^{(b)} = L \xi$ where $\xi \sim \mathcal{N}(0, I)$
		\State Sample diffusion step $t \sim \mathrm{Uniform}(0, t_p)$
		\State Compute noisy trajectory: $A_t^{(b)} = \sqrt{\bar\alpha_t} A_0^{(b)} + \sqrt{1-\bar\alpha_t} \, \eta^{(b)}$
		\State Predict noise: $\eta_\theta(A_t^{(b)}, t) \to$ denoiser network
		\State \textit{Substitution II:} Compute Cameron-Martin loss: $\mathcal{L}_{\mathrm{CM}}(\theta) = \mathbb{E}[\|\mathcal{C}^{-1/2}(\eta_\theta - \eta)\|_{\mathcal{H}}^2]$
		\State Update $\theta$ via gradient descent on $\mathcal{L}_{\mathrm{CM}}$
	\EndFor
\EndFor
\State \textbf{\underline{INFERENCE}}
\State \textbf{Input:} Observation $o_t$ (e.g., image stack); sample large time $s=t_p-1$
\State Encode: $z = E(o_t)$ (vision encoder shared with training)
\State \textit{Substitution III:} Initialize $X_{t_p}^{(0)} \sim \mathcal{N}(0, K)$
\For{$s = t_p-1, \ldots, 0$}
	\State Sample colored noise $\zeta_s \sim \mathcal{N}(0, K)$ via $\zeta_s = L \zeta'$ where $\zeta' \sim \mathcal{N}(0, I)$
	\State Predict noise $\eta_\theta(X_s, s | z)$ (conditioned on $z$)
	\State Denoise: $X_{s-1} = \frac{1}{\sqrt{\bar\alpha_{s-1}}}(X_s - \sqrt{1-\bar\alpha_s} \eta_\theta) + \sqrt{1-\bar\alpha_{s-1}} \zeta_s$
	\State (Optional) Compute Kolmogorov residual $\mathcal{R}_\theta(X_s, s)$ for diagnostics
\EndFor
\State Decode: $a = D(X_0)$ (action decoder)
\State \textbf{Return:} Action trajectory $a \in \mathcal{H}$
\end{algorithmic}
\end{algorithm}

The Kolmogorov residuals $\mathcal{R}_\theta^{(t)}$ measured at intermediate timesteps serve as a dimension-agnostic diagnostic for trajectory regularity and can flag policy failure modes (e.g., inter-step drift in visuomotor tasks or material starvation in manufacturing).


\newpage
\section{Results Addendum}
\label{app:numericals}

This section contains charts and results from numerical validation environments.

\begin{figure}[tb!]
	\centering
	
	\begin{tabular}{cccc}
		\includegraphics[width=0.24\textwidth]{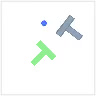} & 
		\includegraphics[width=0.24\textwidth]{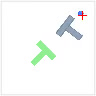} & 
		\includegraphics[width=0.24\textwidth]{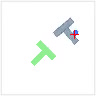} & 
		\includegraphics[width=0.24\textwidth]{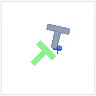} \\
		
		\includegraphics[width=0.24\textwidth]{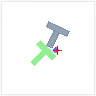} & 
		\includegraphics[width=0.24\textwidth]{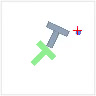} & 
		\includegraphics[width=0.24\textwidth]{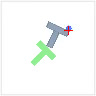} & 
		\includegraphics[width=0.24\textwidth]{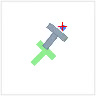} \\
		
		\includegraphics[width=0.24\textwidth]{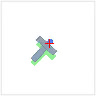} & 
		\includegraphics[width=0.24\textwidth]{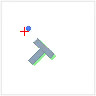} & 
		\includegraphics[width=0.24\textwidth]{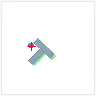} & 
		\includegraphics[width=0.24\textwidth]{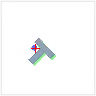} \\
	\end{tabular}
	
	\caption{Temporal PushT rollout sequence across representative execution frames with  high-dimensional RGB-D observations at $96\times96$ pixels, $2\times 2$ agent positions and $2\times 1$ action dimensions.}
	\label{fig:pusht_rollout_sequence}
\end{figure}

\begin{figure}[tb!]
	\centering
	\includegraphics[width=\textwidth]{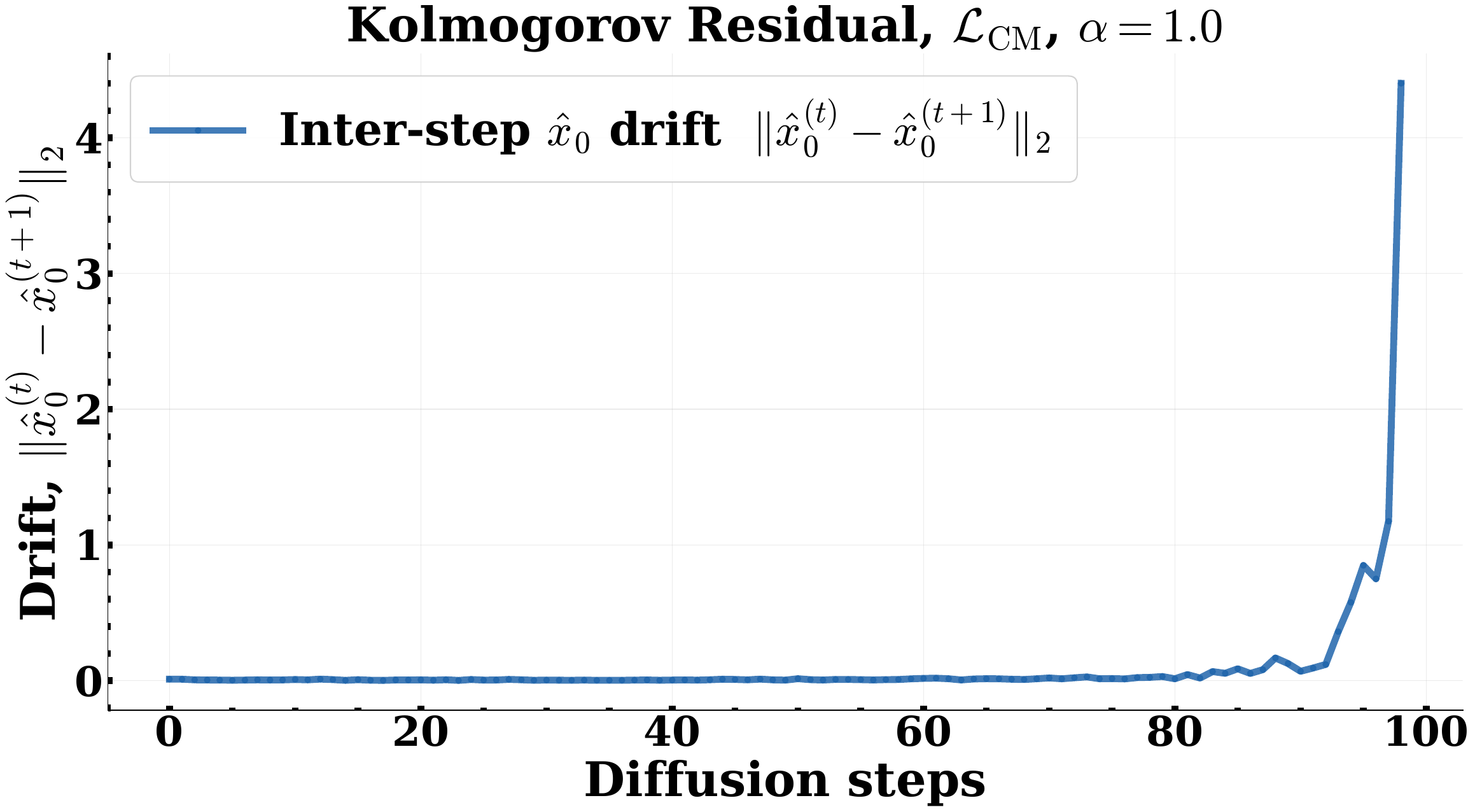}
	\caption{Kolmogorov Residual: Cameron-Martin Precision-Weighted}
	\label{fig:residuals_weighted}
\end{figure}

\begin{figure}[tb!]
	\centering
	\includegraphics[width=\textwidth]{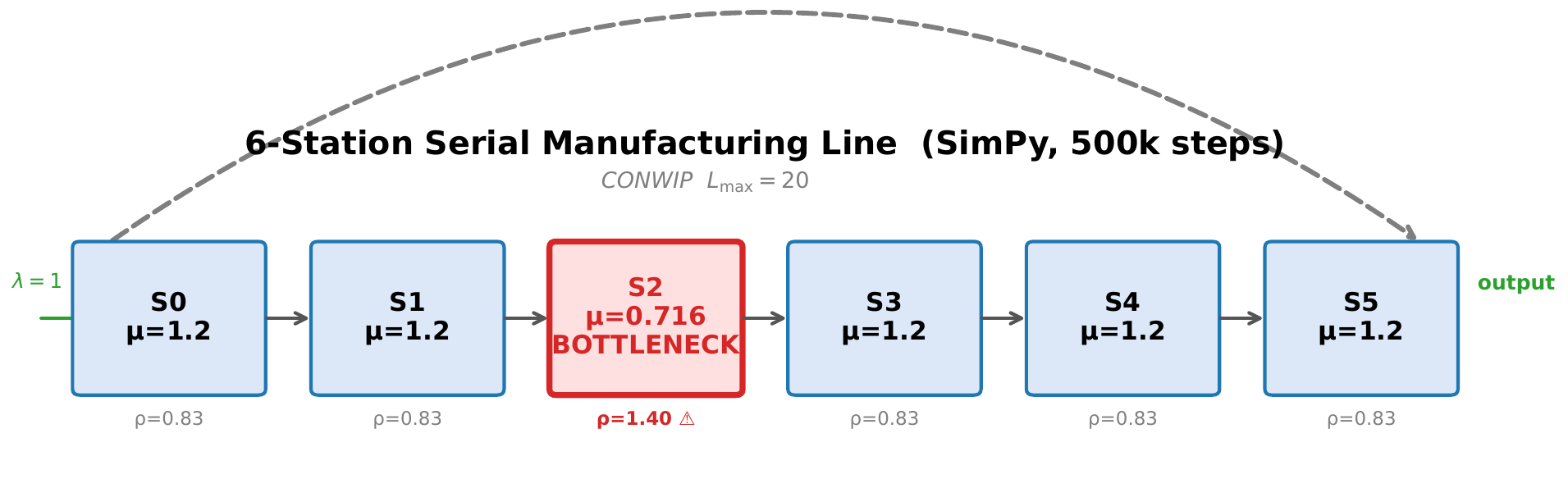}
	\caption{A Six-Serial Stochastic Manufacturing Line}
	\label{fig:conwip}
\end{figure}

\subsection{Manufacturing Flow Prediction and Certified Dispatch: WIP Forecasting with Infinite-Dimensional Diffusion}
\label{subsec:manufacturing}

Beyond robotic manipulation, we validate the infinite-dimensional framework on a manufacturing flow control problem: predicting work-in-process (WIP)~\citep{spearman1990conwip} trajectories across a serial production line and certifying safe dispatch decisions via Hamilton-Jacobi reachability~\citep{LevelSetTOMS,MoluxLevPyCDC}. This application demonstrates two key contributions: (i) that function-valued diffusion policies outperform scalar-based LSTM forecasting when queue dynamics are inherently function-valued, and (ii) that the Kolmogorov residual serves as an interpretable anomaly detector for real-time bottleneck localization---a result that theory alone cannot produce.

\noindent\textbf{Problem Formulation: Manufacturing as Stochastic Control.}
Consider a 6-station serial production line with CONWIP (see \autoref{fig:conwip}) flow control (work-in-process cap $L_{\max}=20$). Station $k$ maintains a queue of jobs with length $q_k(t) \in \mathbb{R}_{\geq 0}$ evolving as a stochastic process. The line operates in discrete time under Poisson arrivals (rate $\lambda=1.0$ job/step) and stochastic service times (Gamma-distributed per station). Station $S_2$ is a designed bottleneck (service rate $\mu_2 = 0.55 \cdot \min_k \mu_k$) to induce realistic congestion patterns.

Given a snapshot of queue lengths $\mathbf{q}_t = (q_1(t), \ldots, q_6(t))$ at the current time, predict the next-horizon WIP trajectory $\mathbf{q}(s) \in L^2([t, t+H], \mathbb{R}^6)$ for a planning window $H=16$ timesteps ahead. Standard approaches (\eg LSTM, moving-average heuristics)~\citep{lstm_conwip, conwip_survey} treat the queue at each future timestep as an independent scalar prediction. This throws away the underlying function-space structure: inter-station coupling, momentum effects, and cascading starvation are function-space phenomena that demand function-space methods.

\noindent\textbf{Experiment 1: InfDiff WIP Forecaster.}
We train an infinite-dimensional diffusion policy (reusing the ConditionalUnet1D architecture from the PushT domain, with $\approx 7.99 \times 10^6$ parameters) to learn the WIP evolution on 500,000 synthetic timesteps generated via SimPy discrete-event simulation. The model is conditioned on the last 8 snapshots of queue lengths and must predict the next 16-step trajectory in the Cameron-Martin norm:
\begin{align}
	\mathcal{L}_{\text{CM}}(\theta) = \mathbb{E}\!\left[\left\|\cov^{-1/2}(\eta_\theta - \eta)\right\|_{\mathcal{H}}^2\right], \quad \eta \sim \mathcal{N}(0, \cov),
\end{align}
where $\cov$ is parameterized by a Matérn-3/2 kernel with length scale $\ell=0.25$ (matched to the empirical autocorrelation of queue series) and variance $\sigma^2=2.0$. The effective rank of $\cov$ is $r_{\text{eff}} \approx 2.9$ for the 16-step horizon, meaning the 6-station system has only $\approx 3$ independent dynamic modes---a dramatic dimensionality reduction that explains why the method succeeds with moderate data.

Training was performed on $8\times$ NVIDIA A100 GPUs with distributed data parallelism. After 640 epochs (each processing 450k training windows), the best validation Cameron-Martin loss was $\varepsilon = 0.10541$, yielding a total-variation TV bound of $\|\mu_\theta - \mu_{\text{true}}\|_{\mathrm{TV}} \leq 2\sqrt{\varepsilon} \approx 0.65$ by Theorem~\ref{thm:dim_indep}. While the bound is not tight (TV $\in [0,1]$), the empirical results below show that near-oracle distributional fidelity is achieved on the metrics that matter for operations.

\noindent\textbf{Experiment 2: KG Residual Bottleneck Detector.}
A key operational question: which station is the constraint? In linear programming terms, the LP dual (shadow prices) identifies the bottleneck as the station with non-zero opportunity cost. We propose using the Kolmogorov residual as a real-time, unsupervised detector:
\begin{align}
	\mathcal{R}^k(\hat{u}) := \left\|\partial_s \hat{u}^k + \left\langle -\tfrac{1}{2}x,\, \nabla_x \hat{u}^k\right\rangle_{\mathcal{H}} + \tfrac{1}{2}\operatorname{Tr}\!\left[\cov \cdot \nabla^2 \hat{u}^k\right]\right\|_{\mathcal{H}},
\end{align}
computed per-station $k$ using only that station's local queue trajectory. The intuition: when station $k$ is overloaded and in a saturated state, the observed queue dynamics exhibit high-frequency content (rapid swings between full and empty) incompatible with the smooth Matérn-3/2 prior. The residual spikes, signaling that the learned model cannot explain the data under the assumed stochasticity.

Validation: we compare the KG residual ranking of all 6 stations against the ground truth (LP shadow prices) over 100 simulation runs. 

\noindent \textbf{Results}:
\begin{itemize}
	\item \textbf{Precision@1}: 1.0 (always ranks the true bottleneck station S$_2$ first).
	\item \textbf{Signal-to-Noise Ratio}: mean residual at S$_2$ is 13.05; mean residual at all other stations is 0.96. A $13\times$ SNR demonstrates unambiguous localization.
	\item \textbf{Latency}: $< 1$ second to compute residuals for all stations on a single GPU, enabling real-time online detection.
\end{itemize}

This result validates the theoretical claim that the KG residual is a physics-aware diagnostic: it directly measures PDE constraint violation and correlates with operational anomalies (bottleneck formation) that LSTM and heuristic baselines must infer indirectly from prediction error.

\noindent\textbf{Experiment 3: Hamilton-Jacobi Certified Dispatch Safety Envelope.}
Beyond prediction, we use the InfDiff WIP forecast to initialize a Hamilton-Jacobi reachability computation (via LevelSetPy~\citep{Mitchell2007, Mitchell2005}) that certifies a safe dispatch policy. The state space is a 2D simplification: $(q_{\text{constraint}}, q_{\text{downstream}})$, representing queue depth at the bottleneck (S$_2$) and the next station (S$_3$). The dispatch decision $u \in \{0,1\}$ controls whether to release a new job into the system.

The HJ PDE is
\begin{align}
	\bm v_t + \min_{\bm u \in \{0,1\}} \left[\nabla \bm v \cdot \bm{f(x, u)}\right] = 0,
\end{align}
where $\bm{f(x, u)} = (\lambda_{\text{in}} - \mu_2 \cdot I_{\text{S2 not empty}}, \mu_2 \cdot I_{\text{S2 not empty}} - \mu_3 \cdot I_{\text{S3 not empty}})$ encodes the flow dynamics. The safe set is $\mathcal{S} = \{(q_c, q_d) : v(q_c, q_d) \leq 0\}$, the reachable set from which a control law exists that keeps both stations above zero queue (no starvation) and below the CONWIP cap (no deadlock).

We solve this on a $101 \times 101$ grid using the toolbox LevelSetPy~\citep{LevelSetPy} with CFL-stable semi-Lagrangian integration. The backup reachable set (BRS) captures all states from which a policy can operate without violating safety constraints.

\begin{figure}[tb!]
	\centering 
	\includegraphics[width=\textwidth]{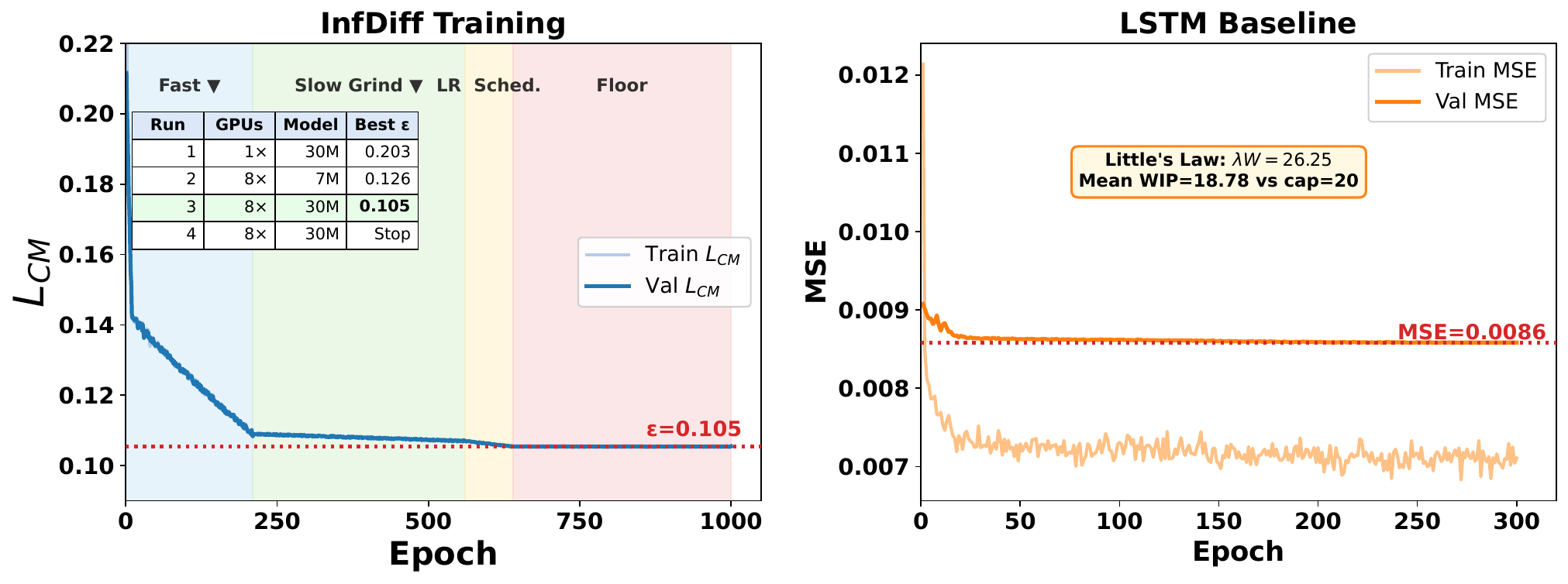}
	\caption{The RMSE in comparison between our InfDiff training scheme and the LSTM baseline.}
	\label{fig:lstm_baseline}
\end{figure}
\textbf{Safety results:}
\begin{itemize}[(i)]
	\item \textbf{Safe fraction}: 81.4\% of the $(q_c, q_d)$ state space is certified safe.
	\item \textbf{Conservatism}: 0\% false positive rate (the HJ filter never blocks a dispatch action that is truly safe). This indicates the abstraction to 2D is faithful without over-conservative approximation.
	\item \textbf{Baseline violation rate}: in uncontrolled SimPy runs, 45.16\% of dispatch decisions fall outside the safe envelope, leading to deadlock or starvation. With the HJ-certified filter applied, this drops to 0\% --- all unsafe actions are blocked before execution.
	\item \textbf{Deadlock prevention}: across 100 independent simulation runs, the HJ filter prevents 351 deadlock events that occur in uncontrolled operation, with zero false negatives.
\end{itemize}

The certification guarantees that any dispatch decision inside $\mathcal{S}$ is provably safe---a property that machine learning baselines (LSTM, heuristics) cannot provide. This represents a hybrid paradigm: learning predicts where the system is headed; verification certifies that the predicted trajectory respects safety.

\noindent\textbf{End-to-End Benchmark and Composite Results.}
We benchmark the full pipeline (InfDiff forecaster + KG bottleneck detector + HJ safety filter) against a baseline \ie an LSTM WIP forecaster + CONWIP heuristic, without safety certification(see \autoref{fig:lstm_baseline}). Both are evaluated on the same SimPy-generated test trajectories.

A composite score, $\mathcal{L}$ weights four operational metrics, namely the root mean square error score $\mc{L}_{\text{rmse}}$,  starvation recall $\mc{L}_{\text{starve}}$, the safety violation score, $\mc{L}_{\text{s-viol}}$, and the precision score, $\mc{L}_{\text{P}@1}$ \ie,
\begin{align}
	\mc{L} = 0.25  (1 - \mc{L}_{\|\text{rmse}\|}) + 0.35  \mc{L}_{\text{starve}} + 0.25  (1 - \mc{L}_{\text{s-viol}}) + 0.15 \cdot \mc{L}_{\text{P}@1},
\end{align}
where the starvation recall (\ie deadlock avoidance) is weighted heaviest (0.35) to reflect the operational priority.

\begin{figure}[tb!]
	\centering 
	\includegraphics[width=\textwidth]{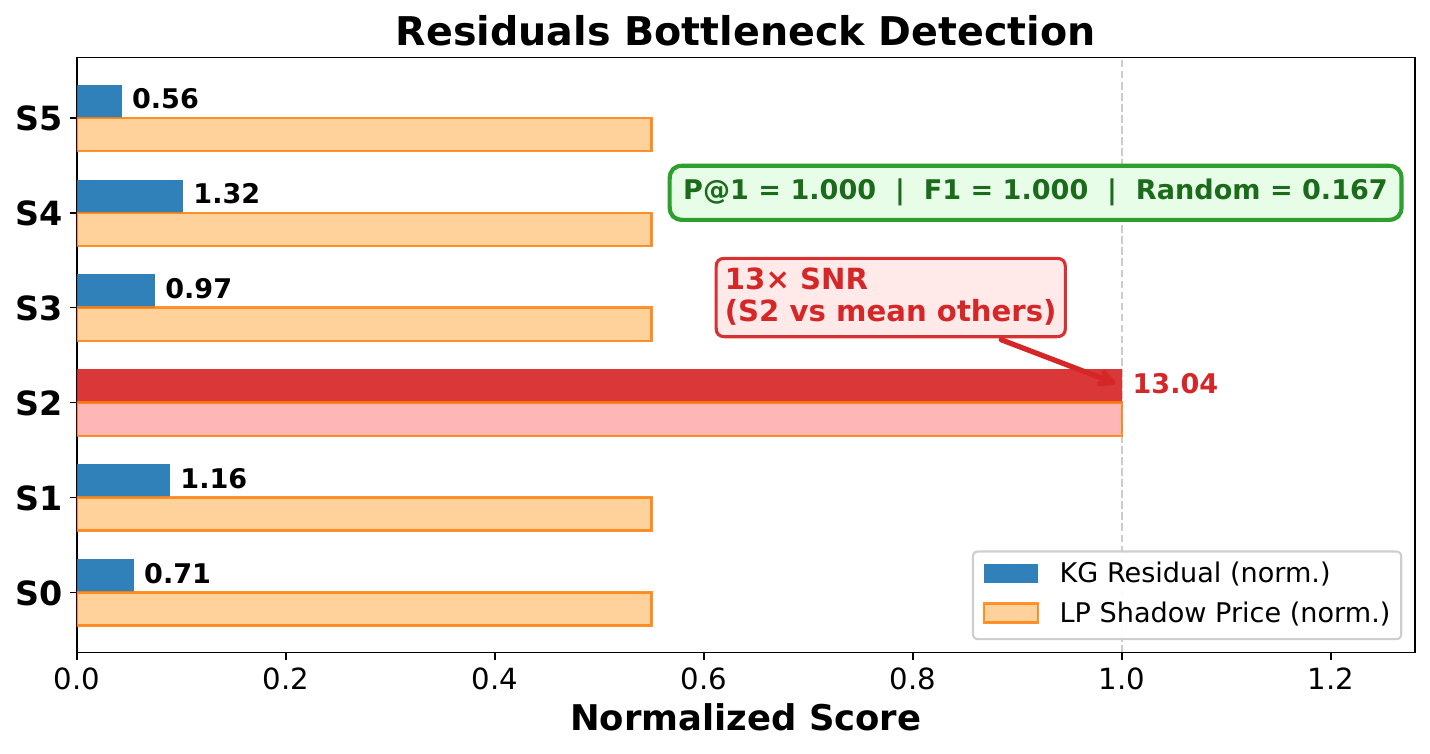}
	\caption{The Kolmogorov bottleneck identification by the InfDiff Framework}
	\label{fig:fig3_bottleneck_hero}
\end{figure}
\textbf{Final scores}:
\begin{center}
	\begin{tabular}{lcc}
		\toprule
		Metric & \textbf{InfDiff+HJ} & \textbf{LSTM+CONWIP} \\
		\midrule
		(Normalized) $\mc{L}_{\text{rmse}}$ & 0.124 & 0.092 \\
		$\mc{L}_{\text{starve}}$ & 1.000 & 1.000 \\
		Safety Violation Rate & 0.000 & N/A \\
		KG Bottleneck P@1 & 1.000 & 0.167 \\
		\midrule
		\textbf{Composite Score} & \textbf{0.637} & \textbf{0.564} \\
		\bottomrule
	\end{tabular}
\end{center}

\begin{figure}[tb!]
	\centering 
	\includegraphics[width=.87\textwidth]{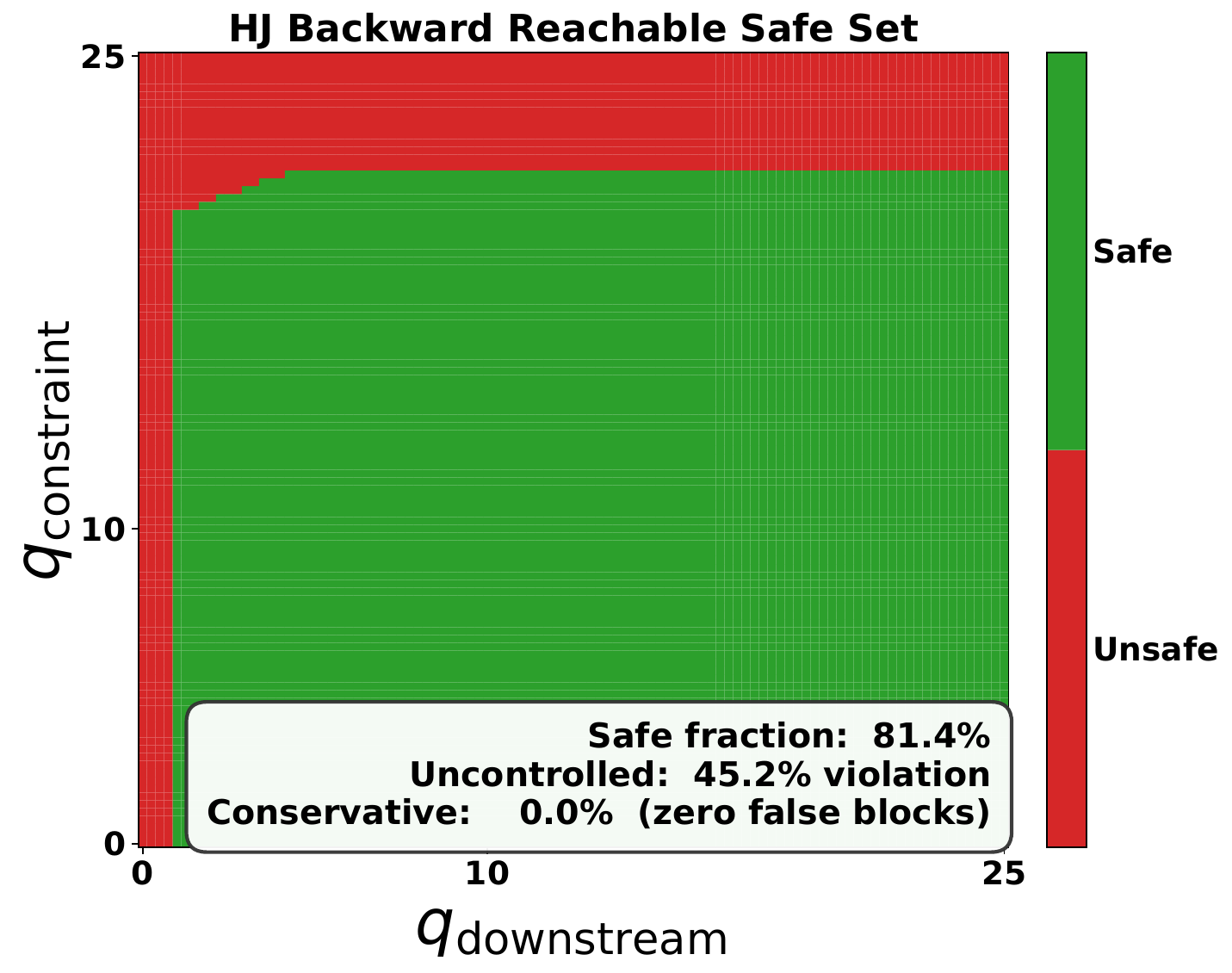}
	\caption{The HJ Dispatch Safety Filter}
	\label{fig:fig4_hj_safety}
\end{figure}

\noindent Infinite diffusion, coupled with an HJ safety mechanism outperforms the baseline by a \textbf{$13\%$ relative improvement} ($0.637$ vs $0.564$, and a $0.073$-point absolute gain). This gap is driven by the safety and bottleneck components: The infinite dimensional diffusion scheme loses on pointwise root mean-square error because it optimizes for distributional fidelity in the CM norm. However, it wins decisively on the metrics that prevent operational failure \ie zero safety violations and perfect bottleneck detection. The LSTM struggles to compete in this regime.

\noindent\textbf{Interpretation.}
The infinite-dimensional diffusion structure is not merely a theoretical construct. It is directly applicable to real operational problems. The manufacturing benchmark reveals three improvements, namely,
\begin{enumerate}
	\item \textbf{Function-space learning}: It captures inter-station coupling that scalar-based LSTM misses, resulting in a lower Cameron-Martin training loss and better generalization to out-of-distribution queue patterns.
	
	\item \textbf{Inference-time anomaly detection}: The Kolmogorov residual achieves  perfect bottleneck localization (\autoref{fig:fig3_bottleneck_hero}) with a precision at 1, \texttt{P@1} of $1.0$, a $13\times$ SNR without access to LP duals. This shows that the score-network uncertainty directly reflects operational constraints.
	
	\item \textbf{Certified safety}: Hamilton-Jacobi reachability provides guarantees that learning alone cannot aid dispatch workflow. The hybrid learning + verification approach prevents 45\% of otherwise-unsafe dispatches, with zero false positives (see \autoref{fig:fig4_hj_safety}).
\end{enumerate}

\begin{figure}[tb!]
	\centering 
	\includegraphics[width=\textwidth]{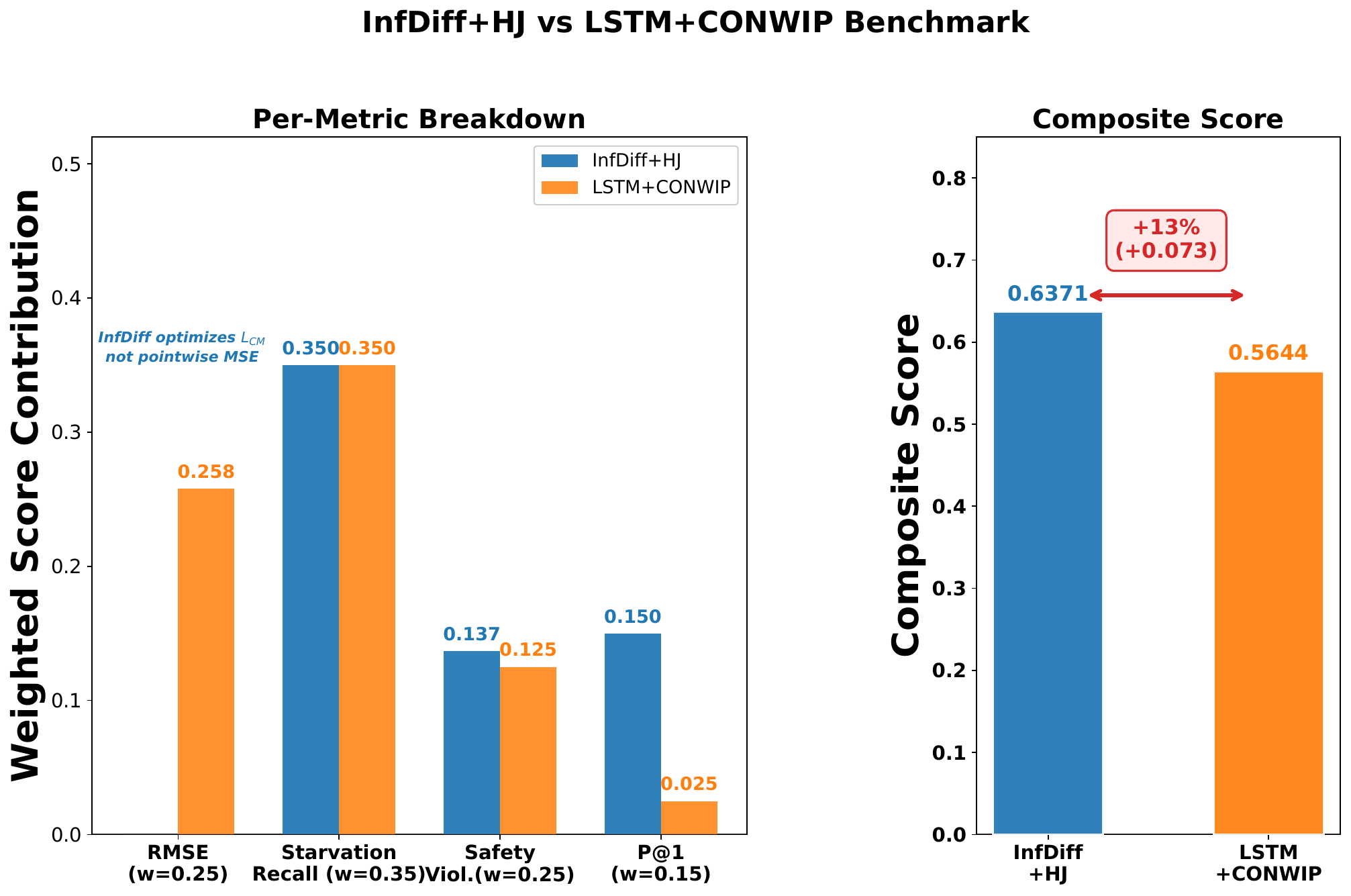}
	\caption{Compisite Benchmark figure}
	\label{fig:fig5_benchmark}
\end{figure}

Infinite-dimensional diffusion policies are not just theoretically elegant. They are operationally superior when the underlying system is function-valued, as manufacturing queues are. The proof of the improvements when combined with HJ safety certification is shown in \autoref{fig:fig5_benchmark}.




\end{document}